\documentclass[10pt,journal,compsoc]{IEEEtran}

\ifCLASSOPTIONcompsoc
  \usepackage[nocompress]{cite}
\else
  \usepackage{cite}
\fi
\usepackage{graphicx}
\usepackage{algorithm}
\usepackage{algorithmic}
\usepackage{multirow}
\usepackage{color}
\usepackage{url}
\usepackage{amsmath,amsfonts,amsthm}
\newtheorem{assumption}{Assumption}
\newtheorem{definition}{Definition}
\newtheorem{lemma}{Lemma}
\newtheorem{theorem}{Theorem}

\usepackage{xcolor}
\usepackage{color, soul}

\begin{document}
%
\title{Achieving Model Fairness in Vertical Federated Learning}
%
%
%
%

\author{
\IEEEauthorblockN{
    Changxin Liu\IEEEauthorrefmark{1}, 
    Zhenan Fan\IEEEauthorrefmark{2},
    Zirui Zhou\IEEEauthorrefmark{3},
    Yang Shi\IEEEauthorrefmark{1},
    Jian Pei\IEEEauthorrefmark{4},
    Lingyang Chu\IEEEauthorrefmark{5} and
    Yong Zhang\IEEEauthorrefmark{3}}\\
\IEEEauthorblockA{
\IEEEauthorrefmark{1} University of Victoria, \texttt{changxin@kth.se, yshi@uvic.ca}\\
\IEEEauthorrefmark{2} University of British Columbia, \texttt{zhenanf@cs.ubc.ca} \\
\IEEEauthorrefmark{3} Huawei Technologies Canada, \texttt{ \{zirui.zhou, yong.zhang3\}@huawei.com} \\
\IEEEauthorrefmark{4} Simon Fraser University, \texttt{ jpei@cs.sfu.ca}\\
\IEEEauthorrefmark{5} McMaster University, \texttt{chul9@mcmaster.ca}
}
\thanks{Changxin Liu and Zhenan Fan contributed equally to this paper.}
}

\IEEEtitleabstractindextext{%
\begin{abstract}
Vertical federated learning (VFL) has attracted greater and greater interest since it enables multiple parties possessing non-overlapping features to strengthen their machine learning models without disclosing their private data and model parameters. Similar to other machine learning algorithms, VFL faces demands and challenges of fairness, i.e., the learned model may be unfairly discriminatory over some groups with sensitive attributes.
To tackle this problem, we propose a fair VFL framework in this work.
First, we systematically formulate the problem of training fair models in VFL, where the learning task is modelled as a constrained optimization problem. 
To solve it in a federated and privacy-preserving manner, we consider the equivalent dual form of the problem and develop an asynchronous gradient coordinate-descent ascent algorithm, where some active data parties perform multiple parallelized local updates per communication round to effectively reduce the number of communication rounds. The messages that the server sends to passive parties are deliberately designed such that the information necessary for local updates is released without intruding on the privacy of data and sensitive attributes. We rigorously study the convergence of the algorithm when applied to general nonconvex-concave min-max problems. We prove that the algorithm finds a $\delta$-stationary point of the dual objective in $\mathcal{O}(\delta^{-4})$ communication rounds under mild conditions. Finally, the extensive experiments on three benchmark datasets demonstrate the superior performance of our method in training fair models.
\end{abstract}

\begin{IEEEkeywords}
vertical federated learning, fairness, min-max optimization
\end{IEEEkeywords}}

\maketitle

\IEEEdisplaynontitleabstractindextext

\section{Introduction}

Federated learning has emerged as a powerful paradigm, where a trustworthy server and multiple organizations collaboratively train a machine learning model for superior performance without intruding the data privacy of any parties~\cite{mcmahan2017communication}. In many practical cases, such as e-commerce, financial and healthcare applications~\cite{hardy2017private,sun2019privacy}, the involved organizations have data about an identical set of subjects but on various attributes. In other words, for every subject, each organization possesses a disjoint partition of the feature vector. Federated learning in such a framework is known as \emph{vertical federated learning} (VFL) and has received increasing attention recently in both academia and industry~\cite{yang2019federated}. 

As a motivating example, suppose a bank initiates a VFL task to train a prediction model for credit score evaluation with an e-commerce company and a social network company.
In this task, the sets of users in these institutions are considered identical but the feature spaces are different.
Particularly, for a set of users,
their revenue and credit rating from the bank, their browsing and purchasing history available in the e-commerce, and their interactions with advertisements recorded by the social network company are collectively used to train the model. 

Similar to other automated decision-making systems, VFL may discriminate the people with certain sensitive attributes (e.g., females, blacks) due to, among other reasons, biased datasets~\cite{caton2020fairness}. 
As preventing sensitive information (e.g., gender, race) from influencing the automated decision-making system ``unfairly'' is crucial for social good, algorithmic fairness has received surging interest in the machine learning community lately~\cite{pessach2020algorithmic}.

It is highly desirable to improve the algorithmic fairness in VFL.
However, designing fair VFL algorithms is challenging due to two characteristics of VFL~\cite{liu2019communication,zhang2021secure}.
First, the data privacy of all the organizations should be fully protected to secure successful collaborations.
This conflicts with the need of a unified training dataset to measure and establish fairness in most of the existing fairness enhancing methods (more discussion in Section~\ref{sec:fair-vfl}). 
Second, the participating organizations in realistic VFL systems typically have imbalanced computational resources and complete their local updates within different time frames.
When training a fair model in VFL, enforcing every organization to launch a single local update per communication round results in inefficiency~\cite{liu2019communication,zhang2021secure}. 

To tackle the aforementioned challenges, in this paper, we develop a comprehensive VFL scheme that achieves a good balance between fairness and accuracy. The key idea is to solve an optimization problem under nonconvex fairness constraints in a distributed manner, for which the local computations are fully parallelized and the communication protocol is privacy-preserving.

We make several major contributions.
First, we systematically formulate the problem of training fair models in VFL (\textit{fair VFL} for short), where the fair learning task is modeled as a nonconvex constrained optimization problem.
To solve it in a federated manner, we consider its equivalent dual form and propose an asynchronous gradient coordinate-descent ascent algorithm to solve the dual problem. 
In the algorithm, some active data parties update local model parameters multiple times in parallel before exchanging information with the server to reduce the number of communication rounds. Moreover, the server masks the necessary information for local updates, and sends the masked version to passive parties to facilitate their local computations while preserving the privacy of data and sensitive attribute.
Under mild conditions, we prove that the algorithm can achieve an $\mathcal{O}(\delta^{-4})$ communication complexity for solving the fair VFL problem.
Finally, we conduct comprehensive experiments to validate the superior performance of the proposed method in training fair models on three benchmark real-world datasets.

The rest of the paper is organized as follows. We review related works in Section~\ref{sec:literature}. We formulate the problem in Section~\ref{sec:problem_formulation}, and develop our solution in Section~\ref{sec:approach}. We present the experimental results in Section~\ref{sec:experiments}, and conclude the paper in Section~\ref{sec:conclusion}.

\section{Related Works}
\label{sec:literature}
In this section, we provide a survey of related works on algorithmic fairness and VFL, followed by a discussion about building fair models within the federated learning framework.

\subsection{Algorithmic Fairness in Machine Learning}
\label{subsec:fairness}
There are two well-developed mathematical definitions for model fairness. First, the {statistical} (or demographic) {parity} refers to the property that the demographics of those receiving a certain outcome should be identical to the demographics of the population overall~\cite{calders2009building}.
Second, equal opportunity (equalized odds) requires that the true positive (and false positive) rates across the demographics be identical~\cite{hardt2016equality}. Recently, a generalization of equal opportunity, a difference of equal opportunities (DEO), is reported in~\cite{donini2018empirical}.

Based on the above definitions for fairness, the existing methods that promote fairness in machine learning can be roughly categorized into three groups, namely, pre-processing methods, post-processing methods, and in-processing methods.

In pre-processing methods~\cite{calmon2017optimized,feldman2015certifying,kamiran2012data,luong2011k}, the training data is refined for fairness reasons. For example, Kamiran and Calders~\cite{kamiran2012data} discussed three pre-processing strategies, namely, massaging, reweighing, and sampling, that promote fairness more efficiently than simply removing the protected attribute from the dataset. A similar method that removes fairness-sensitive features in advance was proposed by Luong et al.~\cite{luong2011k}.
To apply pre-processing methods, all the training data have to be collected in advance for pre-processing.  This may not be achievable in federated learning for privacy reasons.

The post-processing methods re-calibrate a learned model based on its prediction scores within the demographics for fairer predictions~\cite{corbett2017algorithmic,dwork2018decoupled,hardt2016equality,menon2018cost,pleiss2017fairness}. In particular, Hardt et al.~\cite{hardt2016equality} developed a constrained optimization problem for post-processing, where the prediction accuracy can be retained while adjusting the prediction scores within the demographics to enhance fairness. Corbett-Davies~\cite{corbett2017algorithmic} investigated the trade-off between the accuracy and fairness of the model for post-processing methods.
However, since accuracy and fairness requirements are not addressed simultaneously, they typically cannot be balanced well in those methods.

In-processing methods refer to those tailored strategies that explicitly take fairness into account in the training process~\cite{agarwal2018reductions,agarwal2019fair,donini2018empirical,kamiran2010discrimination,kamishima2012fairness,quadrianto2019discovering}. To build fair classifiers directly, Agarwal et al.~\cite{agarwal2018reductions} constructed a linearly constrained optimization problem where the fairness requirement is modeled as inequality constraints. Then, a min-max optimization algorithm is developed to efficiently solve the problem. Alternatively, Quadrianto et al.~\cite{quadrianto2019discovering} formulated an unconstrained optimization problem to learn a representation that still possesses the semantic meaning of the input but is independent of the protected attribute. Most of the existing in-processing algorithms need a unified training dataset and run in a centralized way. Therefore, they are not applicable to the federated learning setting due to both privacy and communication-efficiency concerns.

\subsection{Fair Models and VFL}\label{sec:fair-vfl}
After the seminal work on federated learning~\cite{mcmahan2017communication}, considerable efforts have been dedicated to developing federated learning methods (see~\cite{kairouz2019advances,yang2019federated} for surveys). In this subsection, we focus on VFL and those attempts to train fair models in federated learning.

The standard VFL methods are mostly designed to facilitate privacy preservation or/and efficient parallelized computation~\cite{gascon2016secure,gong2016private,zhang2018feature,hu2019fdml,liu2019communication,chen2020vafl,zhang2021secure}. 
However, none of them considers model fairness, and they cannot be easily extended to train fair models because privacy preservation essentially prohibits collecting the overall dataset, which is a requirement in existing fairness enhancing methods. 
Furthermore, the need for asynchronous parallelized updates in practical VFL setups makes model fairness even more challenging to enhance.

The results on fair federated training are rarely reported. Indeed, only a few existing methods that combine horizontal federated learning (HFL) and fairness are proposed~\cite{mohri2019agnostic,du2021fairness}. Mohri et al.~\cite{mohri2019agnostic} developed the so-called agnostic loss: the maximum of losses defined over an unknown mixture of data distributions.
Upon minimizing the agnostic loss during training, a certain degree of fairness can be expected as a by-product. Following this line, Du et al.~\cite{du2021fairness} addressed the federated fair model training problem explicitly by incorporating an agnostic fairness constraint into the optimization problem for model training. These two works both assumed a horizontally partitioned dataset and cannot be easily extended to VFL.

\section{Problem Formulation}
\label{sec:problem_formulation}
In this section, we review the formal setting of VFL and then formulate the problem of training fair models within VFL.

\subsection{Basic Setup of VFL} \label{sec:setup}
{Consider the following standard VFL scenario: $K$ data parties and a server collaboratively train a machine learning model on $n$ data samples $\{(X_i \in \mathbb{R}^m, y_i \in \{\pm 1\})\}_{i=1}^n$, where $X_i$ is a feature vector and $y_i$ is a label. In VFL, every feature vector $X_i$ is distributed across $K$ data parties, i.e., $X_i = [X_i^1, \dots, X_i^K]$ with $X_i^k \in \mathbb{R}^{m_k}$. We introduce what information each data party has in Section~\ref{sec:fair_model}. }

{
For the theoretical purpose, here we consider the machine learning model in a linear form, i.e., 
\[\sum_{k=1}^K \theta_k^T X_i^k,\]
 where 
 \begin{equation*}
	{ \theta}=[\theta_{1};\ldots;\theta_{K}]
\end{equation*}
and each $\theta_{k} \in\mathbb{R}^{m_k}$ is the block of coordinates associated with data party $k$. In practice, our method also works for more general machine learning models, i.e., 
\[\sum_{k=1}^K g_k(\theta_k, X_i^k),\]
where $g_k$ can be a neural network with weights $\theta_k$. We empirically verify this in Section~\ref{sec:experiments}.
}

To train the model, we consider the following loss function
\begin{equation}\label{def:loss}
	L({ \theta} ) = \frac{1}{n}\sum_{i=1}^{n} l \left( \sum_{k=1}^K \theta_k^T X_i^k, y_i \right)  +\sum_{k=1}^{K}h_k(\theta_{k})  
\end{equation}
where $l:\mathbb{R}\times \mathbb{R}\rightarrow \mathbb{R}$ denotes the loss function, and $h_k:\mathbb{R}^{m_k}\rightarrow \mathbb{R}$ represents the regularizer. The loss function in Equation~\eqref{def:loss} is standard in VFL~\cite{liu2019communication} and has also been used for multi-class classification~\cite{zhang2019asyspa} and regression tasks~\cite{komiyama2018nonconvex} in the literature.

\subsection{Fair Model Training in VFL} \label{sec:fair_model}
We aim to enhance the fairness of the model with respect to a 
\textbf{protected group}, e.g., ``female'' or ``male'', in VFL.
Let $s_i\in\{a,b\}$ represent the membership of $X_i$ between a pair of protected groups.
Based on the samples having a positive label,
we present the definition of difference of equal opportunities (DEO) of the model with respect to the protect group.

\begin{definition}[DEO~\cite{donini2018empirical}] \label{DEO}
	Let
	$$
	\hat{l}^s({ \theta}) = \frac{1}{\lvert \mathcal{N}^{s}\rvert} \sum_{i\in\mathcal{N}^{s}}l \left( {X}_i^T \theta, y_i \right), 
	$$
	where $\mathcal{N}^{s}$ is the set of indexes in which the samples belong to the protected class $s\in\{a,b\}$ and have a positive label.
	The difference of equal opportunities (DEO) of the model is captured by the absolute difference between $ \hat{l}^{a}({ \theta} )$ and $ \hat{l}^{b}({ \theta} )$, i.e.,
 \begin{equation} \label{eq:DEO}
     \mathop{DEO}(\theta) = \lvert \hat{l}^{a}({ \theta} )-\hat{l}^{b}({ \theta} ) \rvert.
 \end{equation}
\end{definition}

A smaller DEO requires the loss function values  associated with the two protected groups to be closer to each other, which further indicates the better fairness of the model with respect to the protected group.
Based on this notion, the following {DEO constraint} is developed by Donini et al.~\cite{donini2018empirical}:
\begin{equation}
	\label{def:fairness_constraint}
	\mathop{DEO}(\theta) \leq \varepsilon,
\end{equation}
where $\varepsilon\geq 0$ represents the maximum unfairness one can tolerate. When Equation~\eqref{def:fairness_constraint} is satisfied by a machine learning model, the model is said to be $\varepsilon$-{fair}.

{
Now we formulate the task of training fair models in VFL. In this task, a group of $K$ parties, under the coordination of a trustworthy {cloud server}, aim to collaboratively train a fair and accurate model without intruding on the data privacy of any party.
We assume two types of data parties, \textit{active party} and \textit{passive party}, where the former refers to those who initiate the task and have the information about the labels, the sensitive attributes, and the loss function; while the latter does not. More specifically, for an active data party $k$, it has the following information 
\[\{X_i^k, y_i, s_i\}_{i=1}^n.\]
For a passive data party $k$, it has the following information
\[\{X_i^k\}_{i=1}^n.\]
The server is assumed to have access to the labels and the sensitive attributes, i.e., 
\[\{y_i, s_i\}_{i=1}^n.\]
Note that we introduce the server for clarity. The server is responsible for updating the dual variables. It can be safely replaced by any active data party without affecting the theoretical results.}

\begin{definition}[Fair VFL Task] 
\label{def:fairVFL}
Given a protected group $\{a,b\}$ and a threshold $\varepsilon>0$ characterizing the maximum unfairness that can be tolerated. The fair VFL task is that a group of $K$ data parties with vertically partitioned data train a machine learning model under the coordination of a trustworthy cloud server such that: i) the data and model parameters of each party are not exposed to any third party and any other data parties, and 
ii) the model is $\varepsilon$-fair.
\end{definition}

Similar to the existing work~\cite{donini2018empirical}, we formulate the above task as a constrained optimization problem
\begin{equation} \label{constrained_VFL}
    \min_{{ \theta}} \quad	L({ \theta}) \enspace\text{s.t.}\enspace \mathop{DEO}(\theta) \leq \epsilon,
\end{equation}
where $L(\theta)$ is the loss function defined in Equation~\eqref{def:loss} and the DEO function is defined in Equation \eqref{eq:DEO}.

\section{Proposed Approach}
\label{sec:approach}
In this section, we solve the fair VFL problem. First, we convert the optimization problem in Equation~\eqref{constrained_VFL} into a nonconvex-concave min-max problem by following Lagrangian relaxation~\cite[Chapter 11]{luenberger1984linear}. Then, we develop an asynchronous gradient coordinate-descent ascent algorithm to solve it without infringing the data privacy of any parties.

\subsection{Lagrangian Relaxation}
Define
\begin{equation*}
{D}({ \theta})  := \hat{l}^{a}({ \theta})-\hat{l}^{b} ({ \theta}).
\end{equation*} 
The fairness constraint in the fair VFL problem can be rewritten as
\begin{subequations}
\begin{align}
	{D}({ \theta})-  \varepsilon &\leq 0 \label{1st_constraint}
	\\	-{D}({ \theta})-	\varepsilon& \leq 	 0. \label{2nd_constraint}
\end{align}
\end{subequations}
We consider the Lagrangian of the constrained problem in Equation~\eqref{constrained_VFL}:
\begin{equation}\label{Lagrangian}
\begin{split}
	f({ \theta},{ \lambda})= {L}({ \theta})+ \lambda_1\left({D}({ \theta})-\varepsilon\right) -\lambda_2 \left(   {D}({ \theta})+\varepsilon\right),
\end{split}
\end{equation}
where $\lambda = [\lambda_1\geq 0; \lambda_2\geq 0]$ are the dual variables associated with the inequality constraints. 
Given a general nonlinear loss function $l$, the Lagrangian is  nonconvex with respect to $\theta$ and concave with respect to $\lambda$. 
Using Equation~\eqref{Lagrangian}, the fair VFL problem can be equivalently transformed to a min-max optimization problem
\begin{equation}\label{min-max}
\min_{{ \theta}}\max_{{ \lambda\in\mathbb{R}_+^2}}f({ \theta},{ \lambda}).
\end{equation}
{Note that the equivalence between problem \eqref{constrained_VFL} and problem \eqref{Lagrangian} is widely known in the literature; see \cite[Proposition~4.3.4]{bertsekas1997nonlinear}. }

\subsection{Asynchronous Min-max Optimization Algorithm}
In this subsection, we develop an asynchronous gradient coordinate-descent ascent algorithm for solving Equation~\eqref{min-max}.

Following~\cite{xu2020unified}, we consider a regularized version of $f$, i.e.,
\begin{equation*}
\tilde{f}_t({ \theta},{ \lambda})=f({ \theta},{ \lambda})-\frac{c_{t}}{2} \lVert \lambda \rVert^2,
\end{equation*}
to speed up the convergence of the algorithm,
where the regularization term $\frac{c_{t}}{2} \lVert  \lambda \rVert^2$ with monotonically decreasing non-negative parameter $\{c_t\}_{t\geq 0}$ 
renders $\tilde{f}_t$ strongly concave with respect to $\lambda$ with modulus ${c_{t}}$. 
The partial gradients of $\tilde{f}_t$ with respect to $\lambda_1$, $\lambda_2$, and $\theta_k, k=1,\cdots, K$ are derived, respectively, as
\begin{equation}\label{def:grad-lambda}
\begin{split}
	\nabla_{\lambda_1}\tilde{f}_t({ \theta},{ \lambda})&=		-c_t \lambda_1+{D}({ \theta}) -\varepsilon \\
	\nabla_{\lambda_2}\tilde{f}_t({ \theta},{ \lambda})&=		-c_t \lambda_2-	{D}({ \theta})-\varepsilon
\end{split}
\end{equation}
and 
\begin{equation}\label{def:partial-grad}
\begin{split}
	\nabla_{k}\tilde{f}_t({ \theta},{ \lambda})=\nabla_{k} {L}({ \theta}) +(\lambda_1-\lambda_2)\nabla_{k} {D}({ \theta})
\end{split}
\end{equation}

To solve Equation~\eqref{min-max} in a distributed manner, each data party $k$ and the server update $\theta_k$ and $\lambda$, respectively.
Notably, computing Equations~\eqref{def:grad-lambda} and~\eqref{def:partial-grad} requires the full information about $\theta$. 
To facilitate local updates, each party sends $\{(X_i)_k^{T}\theta_k^{(t)}\}_{i=1}^n$ to the server at each time $t$, who then calculates $\{X_i^{T}\theta^{(t)}\}_{i=1}^n$ and Equation~\eqref{def:grad-lambda}. 
Based on them, the server performs one projected gradient ascent step to update $\lambda$ by
\begin{equation}\label{dual_update}
\lambda^{(t)} = \left[\lambda^{(t-1)} + \beta \nabla_{{ \lambda}} 
\tilde{f}_{t-1}(\theta^{(t)},\lambda^{(t-1)} )\right]_+
\end{equation}
where $[\cdot]_+$ represents the projection onto the nonnegative orthant.
Then, the server sends $\lambda^{(t)}$ and the other necessary information, specified for active and passive parties in the following, respectively, 
back to the data parties to facilitate their local updates. In particular, since the active parties have the knowledge of labels, protected groups, and loss function, they only require $\lambda^{(t)}$ and $\{X_i^{T}\theta^{(t)}\}_{i=1}^n$ from the server to compute 
\begin{equation}\label{primal_update}
	\theta_k^{(t+1)} = \theta_{ k}^{(t)} - \frac{1}{\eta_t} 	\nabla_{k}\tilde{f}_t({ \theta}^{(t)},{ \lambda^{(t)}}).
\end{equation} 
Since the passive parties do not have the label and sensitive attribute information, they need more information from the server to compute Equation~\eqref{def:partial-grad}.
Consider the following expression for $\nabla_{k}\tilde{f}_t({ \theta},{ \lambda})$:
\begin{equation*}
\begin{split}
    &\nabla_{k}\tilde{f}_t({ \theta},{ \lambda}) \\
    &=  \frac{1}{n}\sum_{i=1}^n \frac{\partial l(X_i^T\theta, y_i)}{\partial(X_i^T\theta)} (X_i)_k  + \sum_{i\in\mathcal{N}^a} {\frac{\lambda_1 -\lambda_2}{\lvert \mathcal{N}^a \rvert} \frac{\partial l(X_i^T\theta, y_i)}{\partial(X_i^T\theta)} }(X_i)_k \\
    & \quad \quad \quad \quad \quad \quad -\sum_{i\in\mathcal{N}^b}{\frac{\lambda_1 -\lambda_2}{\lvert \mathcal{N}^b \rvert} \frac{\partial l(X_i^T\theta, y_i)}{\partial(X_i^T\theta)}}(X_i)_k +  \nabla h_k(\theta_k) \\
     &= \sum_{i=1}^n \nu_i (X_i)_k +  \nabla h_k(\theta_k)
    \end{split}
\end{equation*}
where
	\begin{equation}\label{nu_i}
	\nu_i = \left\{
	\begin{aligned}
		&\left(\frac{1}{n}+\frac{\lambda_1 -\lambda_2}{\lvert \mathcal{N}^a \rvert}\right) \frac{\partial l(X_i^T\theta, y_i)}{\partial(X_i^T\theta )}, \quad i\in\mathcal{N}^a \\
		&\left(\frac{1}{n}-\frac{\lambda_1 -\lambda_2 }{\lvert \mathcal{N}^b \rvert}\right) \frac{\partial l(X_i^T\theta, y_i)}{\partial(X_i^T\theta )}, \quad  i\in\mathcal{N}^b \\
		& \frac{1}{n}\cdot \frac{\partial l(X_i^T\theta, y_i)}{\partial(X_i^T\theta )}, \quad  i\notin \mathcal{N}^a \cup \mathcal{N}^b.
	\end{aligned}
	\right.
\end{equation}
Thus, it is adequate for the server to send $\{\nu_i^{(t)}\}_{i=1}^n$ to the passive parties. 
With them, the passive parties 
are able to update their models according to Equation~\eqref{primal_update}. 
By doing so, we resolve the dilemma of whether the information about the labels and the protected groups should be sent to the passive parties to compute Equation~\eqref{def:partial-grad}. We investigate in Theorem~\ref{thm:security} that disclosing $\{\nu_i\}_{i=1}^n$ preserves the privacy of labels and protected groups.

The updates in Equations~\eqref{dual_update} and~\eqref{primal_update} are performed in an alternating manner~\cite{xu2020unified}, implying that between every two communication rounds each data party updates its local variable once. Nevertheless, in real-world VFL tasks, different data parties typically have imbalanced computational resources and complete their local updates within different time frames.
Enforcing all the parties to launch a single local update between two consecutive communication rounds results in inefficiency~\cite{liu2019communication,zhang2021secure}. 
Therefore, it is highly desirable to enable multiple local updates \textit{in parallel} when solving the fair VFL task.

Motivated by this reason, we allow
each active data party to perform multiple local gradient updates {in parallel} before exchanging information with the server. For passive parties, a single model update is carried out between two consecutive communicating rounds with the server.
The algorithms for the server, active and passive data parties are summarized in Algorithms~\ref{alg:server},~\ref{alg:active} and~\ref{alg:passive}, respectively.

We make the following technical assumption for the number of local iteration rounds performed by active data parties. Such an assumption is standard in federated learning, e.g., standard VFL~\cite{liu2019communication} and HFL~\cite{mcmahan2016federated}.

\begin{assumption}\label{Q-asynchronous}
Between two consecutive communication rounds with the server, each active data party performs updates at least once and at most $Q\geq 1$ times.
\end{assumption}

\begin{algorithm}[t]
	\caption{Fair VFL for Server }\label{alg:server}
	
	\textbf{Input}: Labels and protected classes $\{y_i, s_i \}_{i=1}^n$, parameter $\{c_t\}_{t\geq 0}$, {unfairness tolerance $\varepsilon$, step size $\{\eta_t\}_{t\geq 0}$ and $\beta$.}
	
	\leftline{\textbf{Initialize}: Set $\lambda^{(0)}=0$.}
	\begin{algorithmic}[1] 
		\FOR{$t=1,2\cdots$}
		\STATE Compute $\nabla_{\lambda}\tilde{f}_t({ \theta},{ \lambda})$ using Equation~\eqref{def:grad-lambda}.
		\STATE Update ${ \lambda}\leftarrow  \left[\lambda + \beta \nabla_{{ \lambda}} 
		\tilde{f}_{t-1}(\theta,\lambda )\right]_+$.
		\STATE Receive $\left\{(X_i)_k^T \theta_k \right\}_{i=1}^n$ from each party $k$.
		\STATE Send $\left\{X_i^T\theta \right\}_{i=1}^n$ and ${ \lambda}$ to all the active parties.
		\STATE Send $\{\nu_i\}_{i=1}^n$ to all the passive parties.
		\ENDFOR
	\end{algorithmic}
\end{algorithm}

\begin{algorithm}[t]
\caption{Fair VFL for active data party $k$}
\label{alg:active}

\leftline{\textbf{Input}: Local data $\{({ X}_i)_k, y_i, s_i \}_{i=1}^n$, step size $\{\eta_t\}_{t\geq 0}$.}

\leftline{\textbf{Initialize}: Set $\theta^{(0)}=0$, $\lambda^{(0)}=0$.}
\begin{algorithmic}[1] 
	\FOR{$t=0,1,2,\cdots$}
	\STATE \textit{In parallel for each active party $k$}
		\STATE Receive $\left\{X_i^T\theta \right\}_{i=1}^n$ and ${ \lambda}$ from Server.
	\WHILE {no new information from {Server}}
	\STATE Compute $\left\{ X_i^T{\tilde{\theta}}  \right\}_{i=1}^n$.
	\STATE Compute $\nabla_{k}\tilde{f}_t({{ \tilde{\theta}}},{ \lambda})$.
	\STATE Update $\theta_k \leftarrow \theta_k - {\eta_t^{-1}}\nabla_{k}\tilde{f}_t({\tilde{ \theta}},{ \lambda})$.
	\STATE Send $\left\{(X_i)_k^T {\theta}_k \right\}_{i=1}^n$ to Server.
	\ENDWHILE
	\ENDFOR
\end{algorithmic}
\end{algorithm}

\begin{algorithm}[t]
	\caption{Fair VFL for passive data party $k$}
	\label{alg:passive}
	
\leftline{	\textbf{Input}: Local data $\{({ X}_i)_k \}_{i=1}^n$, step size $\{\eta_t\}_{t\geq 0}$.}
	
\leftline{	\textbf{Initialize}: Set $\theta^{(0)}=0$, $\lambda^{(0)}=0$.}
	\begin{algorithmic}[1] 
		\FOR{$t=0,1,2,\cdots$}
		\STATE \textit{In parallel for each passive party $k$}
		\STATE Receive $\left\{\nu_i \right\}_{i=1}^n$ from Server.
		\STATE Compute $\nabla_{k}\tilde{f}_t({{ \tilde{\theta}}},{ \lambda})$.
		\STATE Update $\theta_k \leftarrow \theta_k - {\eta_t^{-1}}\nabla_{k}\tilde{f}_t({\tilde{ \theta}},{ \lambda})$.
		\STATE Send $\left\{(X_i)_k^T {\theta}_k \right\}_{i=1}^n$ to Server.
		\ENDFOR
	\end{algorithmic}
\end{algorithm}

\subsection{Security Analysis}
In the algorithm, messages bearing intermediate computation results are shared between the server and the data parties, such as the inner product of model parameters and local features $\{ (X_i)_k^T\theta_k\}$, and the weighted partial derivative of the loss $\nu_i$. Note that broadcasting the inner product of model parameters and local features is standard in VFL and has been verified to be privacy-preserving if the dataset and training parameters are undisclosed~\cite{liu2019communication,zhang2021secure}. Thus, we focus on whether sharing $\nu_i$ with the passive parties may leak information about the sensitive attribute and the label. In particular, we consider the following threat model~\cite{cheng2021secureboost,xu2019hybridalpha,zhang2021secure}.

\paragraph{Honest-but-curious} All the data parties follow the algorithm to perform communication and computation. However, they may record the intermediate results to infer the sensitive attribute and the label.

\begin{definition}[Inference attack] 
	An inference attack refers to the behavior that the $k$-th passive party infers the sensitive attribute and the label held by the server.
\end{definition}

\begin{theorem}\label{thm:security}
	Under the honest-but-curious threat model, the proposed algorithm is secure against the inference attack.
\end{theorem}
\begin{proof}[Proof of Theorem \ref{thm:security}]\let\qed\relax
	At each iteration $t$, only $\nu_i^{(t)}$ defined in Equation~\eqref{nu_i} is revealed to each passive party $k$.
	Note that the value of $\nu_i^{(t)}$ is dependent on the tuple $\left(\lambda_1^{(t)}, \lambda_2^{(t)},  \mbox{group}\,\,\mbox{for}\,\,i, \frac{\partial l(X_i^T\theta^{(t)}, y_i)}{\partial(X_i^T\theta^{(t)})}\right).$
	To recover the sensitive attribute from $\nu_i^{(t)}$, the passive parties further need $\lambda_1^{(t)}$, $\lambda_2^{(t)}$, and $\frac{\partial l(X_i^T\theta^{(t)}, y_i)}{\partial(X_i^T\theta^{(t)})}$, which are infeasible.
	Without loss of generality, suppose $i\in\mathcal{N}^a$.
	 Given $\nu_i^{(t)}$, infinite feasible tuples can be constructed, e.g.,  $$\left(\frac{(-\lambda_1^{(t)}+\sigma)\lvert \mathcal{N}^b \rvert}{\lvert \mathcal{N}^a \rvert}, \frac{(\lambda_2^{(t)}-\sigma)\lvert \mathcal{N}^b \rvert}{\lvert \mathcal{N}^a \rvert}, \mathcal{N}^b,  \frac{\partial l(X_i^T\theta^{(t)}, y_i)}{\partial(X_i^T\theta^{(t)})}\right)$$ for some arbitrary $\sigma$, where $\lvert \cdot \rvert$ denotes the cardinality of the set.
	Therefore, the sensitive attribute remains secure regardless of the number of iterations. Since both the form of loss and $X_i^T\theta^{(t)}$ are unknown to the passive data parties, it is also impossible for passive parties to recover the label. 
\end{proof}

\paragraph{Security of the sensitive attribute} Theorem 1 emphasizes that for any passive party $k$ following Algorithm 1, there exist an infinite number of grouping patterns that yield the same set of $\{\nu_i\}_{i=1}^n$. That is, each data party cannot infer the sensitive attribute based on the received messages $\{\nu_i\}_{i=1}^n$ regardless of the number of iterations.

\paragraph{Security of labels} In order to recover the label, the data party should first determine $o_i(t)=\frac{\partial l(X_i^T\theta(t), y_i)}{\partial(X_i^T\theta(t))}$ based on Equation~\eqref{nu_i}, which is not possible. Even with the knowledge of $o_i(t)$, the passive data party still cannot discover the label since the loss form is not available to them.

We remark that, for data samples whose $k$-th blocks are similar, the $k$-th passive data party cannot gain additional information about the sensitive attributes of these samples. The reason is that the other features besides the $k$-th block can be distinct, including the sensitive attribute. Therefore, those data samples lead to different $o_i(t)=\frac{\partial l(X_i^T\theta(t), y_i)}{\partial(X_i^T\theta(t))}$ and do not make the inference procedure easier.

\subsection{Convergence Analysis}

Before establishing the convergence result for the proposed method, we make the following assumption about the smoothness of $f$ defined in Equation~\eqref{Lagrangian}.

\begin{assumption}\label{assum:smoothness}
The function $f(\theta,\lambda)$ is continuously differentiable and there exist constants $L$, $L_{\lambda}$, and $L_{12}$ such that for every $\theta, \theta', \theta''\in\mathbb{R}^m$ and $\lambda, \lambda', \lambda''\in\mathbb{R}^2_+$, we have
\begin{equation*}
	\begin{split}
		\lVert  \nabla_{\theta} f(\theta', \lambda) -\nabla_{\theta} f(\theta'', \lambda)  \rVert & \leq L\lVert \theta' -\theta'' \rVert,\\
		\lVert  \nabla_{\lambda} f(\theta, \lambda') -\nabla_{\lambda} f(\theta, \lambda'')  \rVert & \leq L_\lambda \lVert \lambda' -\lambda'' \rVert,\\
		\lVert  \nabla_{\lambda} f(\theta', \lambda) -\nabla_{\lambda} f(\theta'', \lambda)  \rVert &\leq L_{12}\lVert \theta' -\theta'' \rVert.
	\end{split}
\end{equation*}
\end{assumption}

To proceed, we define the stationarity gap
\begin{equation}\label{stationarity}
\begin{split}
		\nabla G( { \theta}^{(t)},  { \lambda}^{(t)}) = \begin{pmatrix}
		{\eta_t}\left(  { \theta}^{(t)} -{ \theta}^{(t+1)} \right)\\
		\frac{1}{\beta}\left(	 { \lambda}^{(t)} - \left[{ \lambda}^{(t)} + \beta\nabla_{ { \lambda}}f( { \theta}^{(t)},  { \lambda}^{(t)}) \right]_{+} \right)
	\end{pmatrix}.
\end{split}
\end{equation}
Given some target accuracy $\delta>0$, let
$$
{T}(\delta) = \min \left\{ t |\lVert	\nabla {G}( { \theta}^{(t)},  { \lambda}^{(t)})  \rVert \leq {\delta}\right \}.
$$
For the proposed algorithms, we provide a bound on $T(\delta)$ in Theorem~\ref{thm:convergence}, whose proof is postponed to Appendix A.

\begin{theorem}[Convergence Guarantee]\label{thm:convergence}
Suppose that Assumptions~\ref{Q-asynchronous} and~\ref{assum:smoothness} hold. Let $\{(\theta^{(t)}, \lambda^{(t)})\}_{t\geq 0}$ be a sequence generated by Algorithms~\ref{alg:server},~\ref{alg:active}, and~\ref{alg:passive}. If $\beta \geq L_{\lambda}$, $c_t = ({\beta}{{t}^{-1/4}})/2$, and
\begin{small}
	\begin{equation*}
		\begin{split}
			\eta_t \geq \frac{L^2(KQ+2)(KQ-1) + 2(L+1)}{4}+ \frac{L_{12}^2KQ (1+32\tau \sqrt{t})}{2\beta}
		\end{split}
	\end{equation*}
\end{small}
for some $\tau >8$, then
$
	T(\delta) =\mathcal{O}(\delta^{-4})
$
holds for any given $\delta>0$.
\end{theorem}

\section{Experiments}
\label{sec:experiments}
{In this section, we conduct extensive experiments on real-world datasets to validate the fairness and convergence behaviour of our proposed method. Section~\ref{sec:exp_setup} introduces the data sets and general settings of our experiment.}

\subsection{Experiment Setup} \label{sec:exp_setup}

{
As we mentioned in Section~\ref{sec:setup}, we examine the performance of the proposed method with two different machine learning models. The first is the linear model (LM). The loss function $L(\theta)$ can be expressed as 
    \[ 
        \frac{1}{n} \sum_{i=1}^{n} \log\left(1+\exp\left(-y_i\sum_{k=1}^K \theta_k^T X_i^k\right) \right) + \frac{\mu}{2}\sum_{k=1}^K\lVert  \theta_k \rVert^2.
    \]
The second is the multilayer perceptron model (MLP). The loss function $L(\theta)$ can be expressed as 
    \[ 
        \frac{1}{n} \sum_{i=1}^{n} \log\left(1+\exp\left(-y_i\sum_{k=1}^K g(\theta_k, X_i^k)\right) \right) + \frac{\mu}{2}\sum_{k=1}^K\lVert  \theta_k \rVert^2,
    \]
where $g(\theta; \cdot)$ is a MLP model with 2 hidden layers. 
For both models, we test the performance of our proposed method with different fairness levels, i.e. 
\[\epsilon \in \{0.001, 0.01, 0.05, 0.1, 0.5\}.\]
We use FairVFL($\epsilon$) to denote our proposed method with fairness level equal to $\epsilon$.
}

{
We consider $K=6$ data parties where one of them is active, and set $\mu=2/n$. The hyperparameters for fair VFL are set as follows: $c_t = 10^{-3}$, $\eta_t = 100$ and $\beta = 0.1$. 
}

{
We implement our approach in the Julia language~\cite{bezanson2017julia}. Our code is publicly available at \url{https://github.com/ZhenanFanUBC/FairVFL.jl}.
}

\subsubsection{Compared Methods} We use the state-of-the-art VFL method (FedBCD)~\cite{liu2019communication} as a baseline. We also consider two extensions of FedBCD based on the post-processing fairness enhancing strategies, equal opportunity (EO)~\cite{hardt2016equality} and calibrated equal opportunity (CEO)~\cite{pleiss2017fairness}. It is worth mentioning that these two post-processing methods require a unified available training dataset, and thus is not directly implementable in VFL. For comparison reasons, we assume a unified available dataset for these two methods. For all the compared methods, we use the codes published by the authors~\cite{liu2019communication,hardt2016equality,pleiss2017fairness}.

\subsubsection{Datasets} We adopt the following three real-world datasets that are commonly used in literature to verify the performance of fair model training methods~\cite{donini2018empirical,komiyama2018nonconvex}.

The Adult dataset~\cite{kohavi1996scaling} takes an individual's education level, gender, occupation, and some other attributes as features and take whether or not this individual's annual income is above 50,000 dollars as the label. Following the setting considered in~\cite{hardt2016equality}, we use `female' and `male' as the pair of protected groups and use ``above 50,000 dollars'' as the protected class. In the experiments, we uniformly sample 40,000 from 45,222 data instances for training and use the remaining 5,222 data instances for testing. For each sample,  $19$ of $104$ features are assigned to the active data party, and the remaining is evenly distributed to the other $5$ passive parties.

The Compas dataset~\cite{compas} collects 5,278 data instances and contains people's demographic and criminal records. 
Following~\cite{hardt2016equality}, we set `African-American' and `Caucasian' as the pair of protected groups and take ``not a recidivist'' as the protected class. In the experiments, 4,800 data instances are uniformly sampled as the training data and the remaining 478 instances are used as the testing data. For each sample,  $16$ of $26$ features are assigned to the active party, and the rest is evenly assigned to the other $5$ passive parties.

The Crime dataset~\cite{redmond2002data} has 1994 data instances that give socio-economic information and crime rate on communities within the United States.
Following~\cite{kamiran2012data}, we binarize the feature \textit{ViolentCrimesPerPop} based on a threshold of $0.375$ and take the negative as the protected class. We divide the communities according to whether or not the numerical attribute \textit{racepctblack} is larger than $0.06$, and choose the negative as the protected group. We uniformly sample 1,200 data instances as the training data, and use the remaining 794 as the testing data. For each sample, $49$ of $99$ features are assigned to the active data party, and the remaining is evenly divided into $5$ parts, each owned by a passive data party.

\begin{table*}[t]
\centering
\begin{tabular}{|l|lll|lll|lll|}
\hline
\multicolumn{1}{|c|}{\multirow{2}{*}{Method}} & \multicolumn{3}{c|}{Adult}            & \multicolumn{3}{c|}{Compas} & \multicolumn{3}{c|}{Crime} \\ \cline{2-10} 
\multicolumn{1}{|c|}{}                        & \multicolumn{1}{c}{ACC} & DFP  & DFN  & ACC       & DFP    & DFN    & ACC       & DFP    & DFN   \\ \hline 
FedBCD                                        & 85.06\%                 & 0.14 & 0.18 & 69.25\%   & 0.14   & 0.13   & 88.65\%   & 0.17   & 0.27  \\
FedBCD + EO                                   & 73.23\%                 & 0.01 & 0.04 & 66.73\%   & 0.07   & 0.05   & 85.18\%   & 0.07   & 0.19  \\
FedBCD + CEO                                  & 79.89\%                 & 0.10 & 0.14 & 68.20\%   & 0.11   & 0.10   & 86.36\%   & 0.13   & 0.18  \\
FairVFL($\epsilon$=1e-3)                      & 81.41\%                 & 0.01 & 0.08 & 66.95\%   & 0.02   & 0.02   & 85.50\%   & 0.03   & 0.06  \\
FairVFL($\epsilon$=1e-2)                      & 82.11\%                 & 0.01 & 0.09 & 66.95\%   & 0.03   & 0.02   & 85.80\%   & 0.04   & 0.04  \\
FairVFL($\epsilon$=5e-2)                      & 82.77\%                 & 0.03 & 0.12 & 67.78\%   & 0.03   & 0.04   & 86.13\%   & 0.04   & 0.07  \\
FairVFL($\epsilon$=1e-1)                      & 82.36\%                 & 0.07 & 0.13 & 68.39\%   & 0.05   & 0.07   & 87.14\%   & 0.08   & 0.11  \\
FairVFL($\epsilon$=5e-1)                      & 83.08\%                 & 0.14 & 0.16 & 68.36\%   & 0.11   & 0.14   & 88.65\%   & 0.16   & 0.27  \\ \hline
\end{tabular}
\vspace{2mm}
\caption{Experiment results on fairness (linear model). All the methods are evaluated in terms of test accuracy (ACC), the difference between false-positive rates (DFP), and the difference between false-negative rates (DFN).}
\label{exp:fairness_linear}
\end{table*}

\begin{table*}[t]
\centering
\begin{tabular}{|l|lll|lll|lll|}
\hline
\multicolumn{1}{|c|}{\multirow{2}{*}{Method}} & \multicolumn{3}{c|}{Adult}            & \multicolumn{3}{c|}{Compas} & \multicolumn{3}{c|}{Crime} \\ \cline{2-10} 
\multicolumn{1}{|c|}{}                        & \multicolumn{1}{c}{ACC} & DFP  & DFN  & ACC       & DFP    & DFN    & ACC       & DFP    & DFN   \\ \hline
FedBCD                                        & 85.55\%                 & 0.14 & 0.19 & 70.36\%   & 0.15   & 0.13   & 89.65\%   & 0.17   & 0.31  \\
FedBCD + EO                                   & 77.23\%                 & 0.01 & 0.03 & 66.78\%   & 0.08   & 0.09   & 86.13\%   & 0.07   & 0.14  \\
FedBCD + CEO                                  & 81.39\%                 & 0.10 & 0.04 & 64.44\%   & 0.11   & 0.13   & 88.36\%   & 0.15   & 0.07  \\
FairVFL($\epsilon$=1e-3)                      & 81.98\%                 & 0.00 & 0.11 & 68.41\%   & 0.01   & 0.00   & 84.62\%   & 0.01   & 0.06  \\
FairVFL($\epsilon$=1e-2)                      & 81.98\%                 & 0.01 & 0.11 & 68.83\%   & 0.02   & 0.03   & 85.87\%   & 0.03   & 0.06  \\
FairVFL($\epsilon$=5e-2)                      & 82.17\%                 & 0.04 & 0.12 & 68.62\%   & 0.03   & 0.03   & 86.19\%   & 0.05   & 0.05  \\
FairVFL($\epsilon$=1e-1)                      & 82.55\%                 & 0.07 & 0.13 & 70.29\%   & 0.05   & 0.05   & 88.13\%   & 0.08   & 0.12  \\
FairVFL($\epsilon$=5e-1)                      & 83.15\%                 & 0.14 & 0.13 & 70.36\%   & 0.12   & 0.12   & 89.01\%   & 0.15   & 0.28  \\ \hline
\end{tabular}
\vspace{2mm} 
\caption{Experiment results on fairness (MLP model). All the methods are evaluated in terms of test accuracy (ACC), the difference between false-positive rates (DFP), and the difference between false-negative rates (DFN).}
\label{exp:fairness_mlp}
\end{table*}

\begin{figure*}[t]
    \centering
    \includegraphics[width=.8\textwidth]{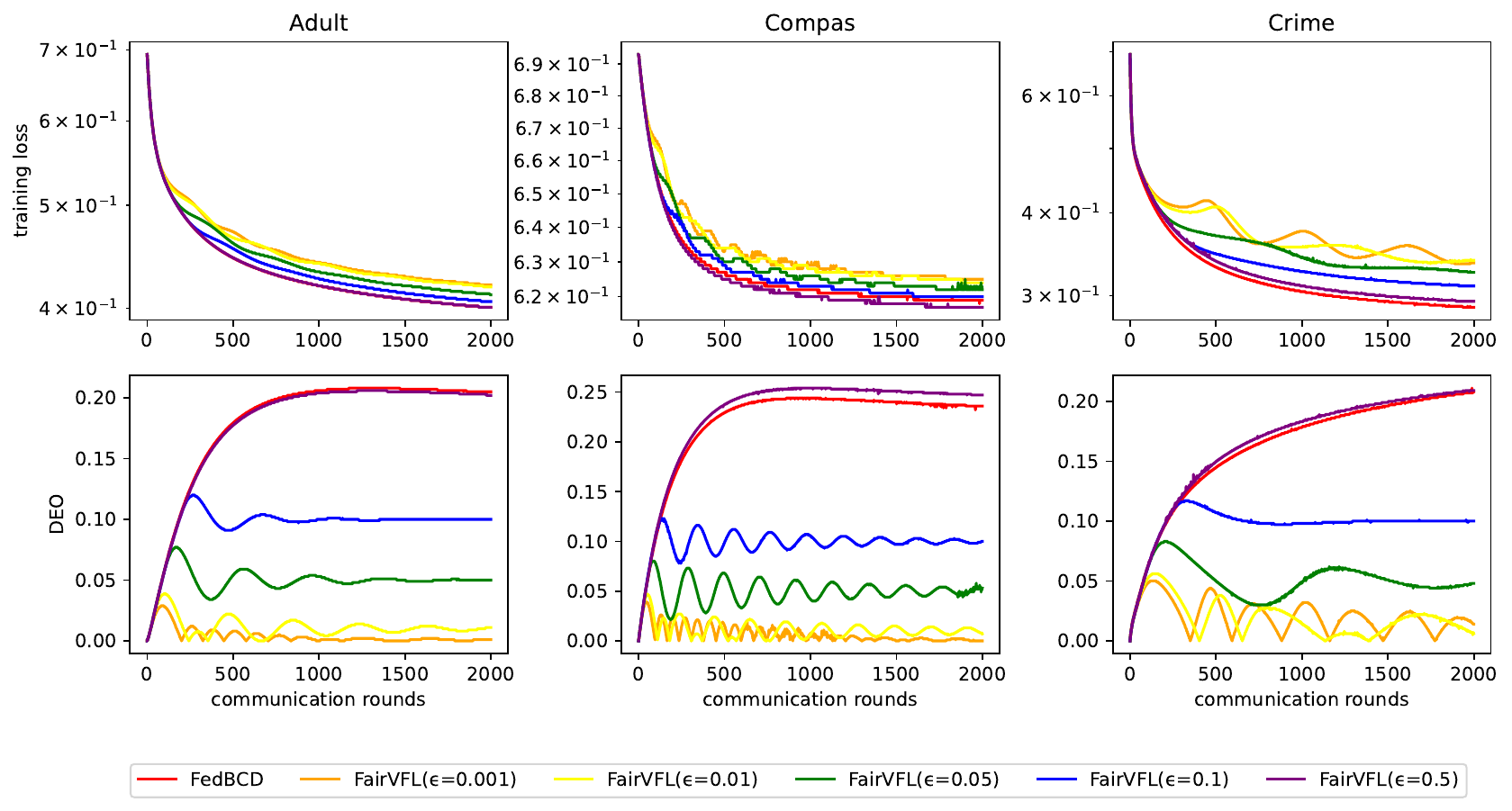}
    \caption{Experiment results on convergence behaviour (linear model). All the methods are evaluated in terms of training objective $L(\theta)$ and the constraint $\mathop{DEO}(\theta)$.}
    \label{fig:convergence_linear}
\end{figure*}

\begin{figure*}[t]
    \centering
    \includegraphics[width=.8\textwidth]{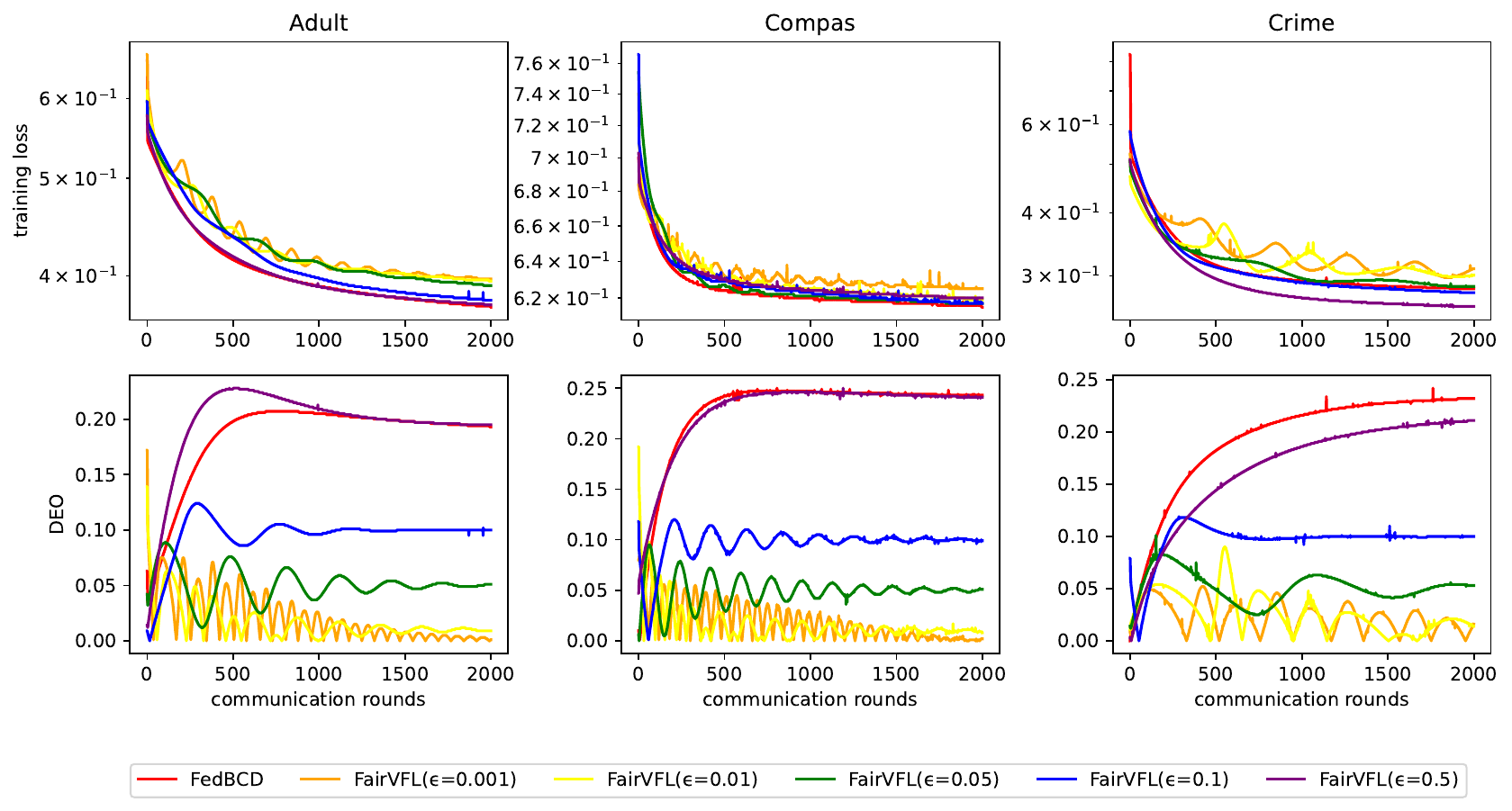}
    \caption{Experiment results on convergence behaviour (MLP model). All the methods are evaluated in terms of training objective $L(\theta)$ and the constraint $\mathop{DEO}(\theta)$.}
    \label{fig:convergence_mlp}
\end{figure*}

\subsection{Experiment results on fairness}
{
In this set of experiments,  we want to check the performance on fairness for our proposed method FairVFL with different fairness levels $\epsilon$, and the compared methods: FedBCD, FedBCD + EO and FedBCD + CEO, where the models for FedBCD + EO and FedBCD + CEO are post-processed based on the model trained by FedBCD, respectively. Given two protected groups $\{a, b\}$, fairness methods aim to ensure that no error rate disproportionately affects any group. In other words, both groups should have a similar false-positive rate, and both groups should have a similar false-negative rate. More specifically, in this experiment, we compare the difference between false-positive rates (DFP) and the difference between false-negative rates (DFN) of the two protected groups. A fair method is expected to have low DFP and DFN. Besides, we also compare the test accuracies (ACC) of all the methods. The results are shown in Table~\ref{exp:fairness_linear} and Table~\ref{exp:fairness_mlp}, where Table~\ref{exp:fairness_linear} contains the results for LM and Table~\ref{exp:fairness_mlp} contains the results for MLP. }

{
As shown in both tables, the model trained by FedBCD cannot achieve good fairness because it ignores fairness during training. The post-processing methods, EO and CEO, typically improve the fairness of the models by FedBCD at the expense of sacrificing prediction accuracy. Notably, EO and CEO both cannot be implemented in a federated manner, because they only apply to a unified available dataset, which intrudes on the data privacy of the data parties. The models trained by fair VFL achieve comparable accuracy with the FedBCD in most cases and demonstrate better fairness in all cases. Besides, there is a trade-off with the fairness level $\epsilon$. The accuracy of
the models trained on the three datasets witnesses a gradual decrease when $\varepsilon$ becomes smaller because reducing $\varepsilon$ renders a tighter DEO constraint and shrinks the feasible region of the fair model training problem.  This suggests that in practice we can use techniques like cross-validation to choose an appropriate fairness level $\epsilon$.
}

\subsection{Experiment results on convergence behaviour}
{
In this set of experiments, we analyze the convergence behaviour of our proposed method on solving Problem~\eqref{constrained_VFL}. So we plot the training objective $L(\theta)$ and constraint $\mathop{DEO}(\theta)$ versus the communication rounds. The results are shown in Figure~\ref{fig:convergence_linear} and Figure~\ref{fig:convergence_mlp}, where Figure~\ref{fig:convergence_linear} contains the results for LM and Figure~\ref{fig:convergence_mlp} contains the results for MLP.
}

{First, we consider the convergence behaviour of the training objective $L(\theta)$. Let's define a value function $v:\mathbb{R}_+ \to \mathbb{R}$ as
\[v(\epsilon) = \min_{\theta} \{L(\theta)\mid \mathop{DEO}(\theta) \leq \epsilon\}.\]
It is easy to see that $v$ is monotonically non-increasing. As we can see from both figures, the larger the fairness level $\epsilon$, the lower value $L(\theta)$ will converge to, which agrees with our analysis of the value function. Besides, when the $\epsilon$ is large enough, the convergence behaviour is similar to FedBCD, which suggests that the constraint is loose. Next, we consider the convergence behaviour of the constraint function $\mathop{DEO}(\theta)$. As we can see from both figures, when the fairness level $\epsilon$ is not too large, $\mathop{DEO}(\theta)$ indeed converges to $\epsilon$. This observation numerically supports the equivalence between problem~\eqref{constrained_VFL} and problem~\eqref{Lagrangian}. When the fairness level $\epsilon$ is too large, $\mathop{DEO}(\theta)$ agrees with the value obtained by FedBCD, which also suggests that the constraint is not playing a role in this case.}

\section{Conclusion}
\label{sec:conclusion}
In this work, we tackle the problem of training fair models in VFL.
To promote fairness in VFL, we incorporate a nonconvex DEO constraint into the optimization problem for training.

To solve the constrained problem in a distributed way, we consider its equivalent dual form and develop an asynchronous min-max optimization algorithm, where each party continuously performs parallelized local updates per communication round to improve communication efficiency.
 
To preserve data privacy, we develop a masking strategy that releases necessary information for local computations without disclosing data and sensitive attributes.
We theoretically prove the convergence of the proposed algorithm.  

The current work opens up new revenues for future research. For example, it may be worthwhile to further strengthen the privacy of the information-releasing mechanism within server using tools from differential privacy, such that the mechanism becomes resistant to privacy attacks based on auxiliary information. Illustrating the potential gain from asynchronous local updates in nonconvex-concave optimization via theoretical evidence is also interesting.

\appendices




\section{Supporting Lemmas and Their Proofs}\label{Supple}

In Appendix \ref{Supple}, we introduce a labeling strategy for the variables, and present three technical lemmas that are used to prove Theorem \ref{thm:convergence} in Appendix \ref{app:pf-thm}.

\textbf{Labeling strategy.} When allowing multiple local updates in parallel, each data party has inconsistent read of the actual model parameter.
Between two consecutive communication rounds at $t$ and $t+1$,
we define by $\theta^{(t,\tau)}$ the actual model parameter with ${\theta}^{(t,0)}= \theta^{(t)}$. Let $\psi(t,\tau)$ be the index of the data party that performs update at $(t,\tau)$. The local inconsistent read of $\theta^{(t,\tau)}$ by party $\psi(t,\tau)$ is written as
\begin{equation*}
\tilde{\theta}^{(t,\tau)} =[\theta_{1}^{(t)};\cdots;\theta_{\psi(t,\tau)-1}^{(t)};\theta_{\psi(t,\tau)}^{(t,\tau)};\theta_{\psi(t,\tau)+1}^{(t)}; \cdots;\theta_{K}^{(t,\tau)}].
\end{equation*}
Based on this labeling strategy, the local updates of the data parties can be summarized as 
\begin{equation}\label{overall_update}
\theta^{(t,\tau+1)} = 	\theta^{(t,\tau)} - \frac{1}{\eta_t}{ U}_{\psi(t,\tau)} \tilde{g}^{(t,\tau)}
\end{equation} 
where $U_{\psi(t,\tau)}\in\mathbb{R}^{m\times m_i}$,  $[U_1,\cdots,U_q]=I_m$ and 
\begin{small}
\begin{equation*}
	\tilde{g}^{(t,\tau)} =  \nabla\tilde{f}_t(\tilde{\theta}^{(t,\tau)},{ \lambda^{(t)}}).
\end{equation*}
\end{small}
Let $\kappa(t)$ represents the number of updating times and $\theta^{(t+1)} = \theta^{(t,\kappa(t))}$.

\begin{lemma}[Primal Progress]\label{lem:primal}
	Suppose Assumption \ref{assum:smoothness} holds. Then, for all $t\geq 0$, it holds that
	\begin{equation*}
		\begin{split}
			f({ \theta}^{(t+1)}, { \lambda}^{(t)}) - 		f({ \theta}^{(t)}, { \lambda}^{(t)})
			\leq  -\iota_t\sum_{\tau=0}^{\kappa(t)-1}  \left \lVert  { \theta}^{(t,\tau+1)}- { \theta}^{(t,\tau)}  \right\rVert^2
		\end{split}
	\end{equation*}
	where $$\iota_t = \eta_t-\frac{L+1}{2} - \frac{ L^2(KQ+2)(KQ-1)}{4}.$$
\end{lemma}

\begin{proof}[Proof of Lemma \ref{lem:primal}]\let\qed\relax
	We start by considering
	\begin{equation*}
		\begin{split}
			&f({ \theta}^{(t,\tau+1)}, { \lambda}^{(t)}) - 		f({ \theta}^{(t,\tau)}, { \lambda}^{(t)})  \\
			&\overset{(i)}{\leq}  \left \langle \nabla_{{ \theta}}f({ \theta}^{(t,\tau)}, { \lambda}^{(t)}),  { \theta}^{(t,\tau+1)}- { \theta}^{(t,\tau)}   \right\rangle \\
			& \quad + \frac{L}{2}\left\lVert { \theta}^{(t,\tau+1)}-{ \theta}^{(t,\tau)} \right\rVert^2  \\
			&\overset{(ii)}{\leq}  		   -\left( \eta_t-\frac{L+1}{2}\right) \left\lVert  { \theta}^{(t,\tau+1)}- { \theta}^{(t,\tau)}  \right\rVert^2  \\
			&\quad + \frac{1}{2}\left\lVert    \left(  \nabla_{{ \theta}}f({ \theta}^{(t,\tau)}, { \lambda}^{(t)}) - \tilde{g}^{(t,\tau)}\right)_{\psi(t,\tau)} \right\rVert^2   
		\end{split}
	\end{equation*}
	where in $(i)$ we use the Lipschitz continuity of the gradient of $f$ and in $(ii)$ the Young's inequality.
Using the Lipschitz continuity of the gradient of $f$ that
$
			\left\lVert    \left(  \nabla_{{ \theta}}f({ \theta}^{(t,\tau)}, { \lambda}^{(t)}) - \tilde{g}^{(t,\tau)}\right)_{\psi(t,\tau)} \right\rVert^2      \leq {L^2}\left\lVert   \theta^{(t,\tau)} - \tilde{\theta}^{(t,\tau)}   \right\rVert^2 ={L^2}  \sum_{k\neq  \psi(t,\tau)}\left\lVert  \left( \theta^{(t,\tau)} - {\theta}^{(t)}  \right)_k  \right\rVert^2,
$
	we obtain
	\begin{small}
	\begin{equation*}
		\begin{split}
			&f({ \theta}^{(t,\kappa(t))}, { \lambda}^{(t)}) - 		f({ \theta}^{(t)}, { \lambda}^{(t)}) \\
			& = \sum_{\tau=0}^{\kappa(t)-1} f({ \theta}^{(t,\tau+1)},{ \lambda}^{(t)}) - 		f({ \theta}^{(t,\tau)}, { \lambda}^{(t)}) \\
			& \leq   -\left(\eta_t-\frac{L+1}{2}\right)\sum_{\tau=0}^{\kappa(t)-1}  \left \lVert  { \theta}^{(t,\tau+1)}-{ \theta}^{(t,\tau)} \right \rVert^2 \\
			&\quad +   \frac{ {L^2} }{2}\sum_{\tau=1}^{\kappa(t)-1}   \sum_{k\neq \psi(t,\tau)}\left\lVert  \sum_{m=0 }^{\tau-1} \left({ \theta}^{(t,m+1)} -{ \theta}^{(t,m)} \right)_k \right\rVert^2 .
		\end{split}
	\end{equation*}
	\end{small}
	Since $   \left({ \theta}^{(t,m+1)} -{ \theta}^{(t,m)}  \right)_k   = 0$ when $k\neq \psi(t,m)$,
	and
$
		K\leq \kappa(t) \leq KQ,
$
	we have
	\begin{equation*}
		\begin{split}
			&	f({ \theta}^{(t,\kappa(t))}, { \lambda}^{(t)}) - 		f({ \theta}^{(t)}, { \lambda}^{(t)})\\
			&\leq -\left(\eta_t-\frac{L+1}{2}\right)\sum_{\tau=0}^{\kappa(t)-1}  \left \lVert  { \theta}^{(t,\tau+1)}- { \theta}^{(t,\tau)}  \right\rVert^2\\
			&\quad +  \frac{{L^2}(KQ+2)(KQ-1)}{4} \sum_{\tau=0}^{\kappa(t)-1}\left\lVert  \theta^{(t,\tau+1)} - {\theta}^{(t,\tau)}   \right\rVert^2    \\
			& =  -\iota_t\sum_{\tau=0}^{\kappa(t)-1}  \left \lVert  { \theta}^{(t,\tau+1)}- { \theta}^{(t,\tau)}  \right\rVert^2.
		\end{split}
	\end{equation*} \end{proof}

\begin{lemma}[Dual Progress]\label{lem:dual-progress}
	Suppose Assumption \ref{assum:smoothness} holds. If
\begin{equation}\label{condition:beta}
 \beta \geq \frac{\tilde{L}_\lambda +c_1}{2}
\end{equation} 
where $\tilde{L}_\lambda = L_\lambda  + c_1$,
then, for all $t\geq 0$ and $a_t > 0$, it holds  that
	\begin{equation}
		\begin{split}
		&	{f}({ \theta}^{(t+1)},{ \lambda}^{(t+1)})-{f}({ \theta}^{(t+1)},{ \lambda}^{(t)}) \\
			&\leq \frac{L_{12}^2\kappa(t)}{2a_t}  \sum_{\tau=0}^{\kappa(t)-1}\left\lVert { \theta}^{(t,\tau+1)} -{ \theta}^{(t,\tau)}\right\rVert^2 \\
			&\quad +  \left(\frac{\beta}{2}-\frac{c_{t-1}-c_t}{2}+\frac{a_t}{2}\right) \left\lVert { { \lambda}}^{(t+1)}- { \lambda}^{(t)}  \right\rVert^2  \\
			&\quad + \frac{c_{t-1}}{2}\left(\left\lVert { \lambda}^{(t+1)} \right\rVert^2 - \left\lVert { \lambda}^{(t)} \right\rVert^2  \right) + \frac{\beta}{2} \left\lVert { { \lambda}}^{(t)}- { \lambda}^{(t-1)}  \right\rVert^2.
		\end{split}
	\end{equation}
\end{lemma}
\begin{proof}[Proof of Lemma \ref{lem:dual-progress}]\let\qed\relax
 When there is no ambiguity, denote  by $\cdot^{(+)}$, $\cdot$, $\cdot^{(-)}$ the variables at time $t+1$, $t$, and $t-1$, respectively.
 	Recall
$
		\tilde{f}_t({ \theta},{ \lambda})=f({ \theta},{ \lambda})-\frac{c_{t}}{2}\lVert{ \lambda} \rVert^2.
$
Letting $\tilde{L}_{{ \lambda}}= {L_{{ \lambda}}+c_{1}}$, by Assumption \ref{assum:smoothness}, we have
$
		\left\lVert	\nabla_{{ \lambda}}\tilde{f}_{t-1}({ \theta},{ \lambda})-\nabla_{\lambda}\tilde{f}_{t-1}({ \theta},{ \lambda}^{(-)}) \right\rVert \leq \tilde{L}_{{ \lambda}} \left\lVert{ \lambda}-{ \lambda}^{(-)}\right\rVert.
$
	By the strong concavity of $\tilde{f}_t(\theta,\lambda)$ with respect to $\lambda$, we have \cite[Theorem 2.1.12]{nesterov2003introductory}
	\begin{equation}\label{ineq:smooth-strong concavity}
		\begin{split}
			&\left\langle  \nabla_{{ \lambda}}\tilde{f}_{t-1}({ \theta},{ \lambda}) -\nabla_{{ \lambda}}\tilde{f}_{t-1}({ \theta},{ \lambda}^{(-)}), { \lambda}-{ \lambda}^{(-)}\right\rangle \\
			& \leq - \frac{c_{t-1}\tilde{L}_{{ \lambda}}}{c_{t-1}+\tilde{L}_{{ \lambda}}} \left \lVert  { \lambda}-{ \lambda}^{(-)} \right\rVert^2  \\
			& \quad   -\frac{1}{\tilde{L}_{{ \lambda}}+c_{t-1}} \left\lVert  \nabla_{{ \lambda}}\tilde{f}_{t-1}({ \theta},{ \lambda}) -\nabla_{{ \lambda}}\tilde{f}_{t-1}({ \theta}, { \lambda}^{(-)})\right\rVert^2 .
		\end{split}
	\end{equation}
Due to
$
 \lambda = \left[\lambda^{(-)} + \beta \nabla_{{ \lambda}} 
\tilde{f}_{t-1}(\theta,\lambda^{(-)} )\right]_+,
$
we have, by optimality, that
	\begin{equation}
	\left\langle \nabla_{{ \lambda}} \tilde{f}_{t-1}({ \theta} , { \lambda}^{(-)})- \beta ({ \lambda}- { \lambda}^{(-)}), \hat{ \lambda} -{ \lambda} \right\rangle \leq 0, \,\, \forall \hat{\lambda}. \label{dual_optimality_t}
	\end{equation}
	Using \eqref{dual_optimality_t}, we have
	\begin{equation}\label{decreasing-modified-dual}
		\begin{split}
			&\tilde{f}_{t}({ \theta}^{(+)},{ \lambda}^{(+)})-\tilde{f}_t({ \theta}^{(+)},{ \lambda}) \leq \left\langle  \nabla_{{ \lambda}}\tilde{f}_t({ \theta}^{(+)},{ \lambda}),  { \lambda}^{(+)}-{ \lambda} \right\rangle \\
			&\leq      \left\langle  \nabla_{{ \lambda}}\tilde{f}_t({ \theta}^{(+)},{ \lambda})- \nabla_{{ \lambda}}\tilde{f}_{t-1}({ \theta},{ \lambda}^{(-)}),  { \lambda}^{(+)}- { \lambda} \right\rangle  \\
			& \quad +  \beta \left\langle  { \lambda}- { \lambda}^{(-)},   { \lambda}^{(+)}- { \lambda} \right \rangle .
		\end{split}
	\end{equation}
For the second term on the right-hand side of \eqref{decreasing-modified-dual}, we have
	\begin{equation*}
		\begin{split}
		&	2 \left\langle { \lambda}-{ \lambda}^{(-)},  { \lambda}^{(+)}- { \lambda} \right\rangle \\
			&=  \left\lVert { \lambda}- { \lambda}^{(-)} \right\rVert^2 - \left\lVert      { \lambda}^{(+)}- { \lambda} -\left ({ \lambda}-{ \lambda}^{(-)}\right)   \right\rVert^2 +\left\lVert   { \lambda}^{(+)}- { \lambda} \right\rVert^2 
		\end{split} 
	\end{equation*}
For the first term on the right-hand side of \eqref{decreasing-modified-dual}, we consider
\begin{small}
	\begin{equation*}
		\begin{split}
			&\left\langle  \nabla_{{ \lambda}}\tilde{f}_t({ \theta}^{(+)},{ \lambda})- \nabla_{{ \lambda}}\tilde{f}_{t-1}({ \theta},{ \lambda}^{(-)}),  { \lambda}^{(+)}- { \lambda} \right\rangle \\
		&	=  \underbrace{\left\langle \nabla_{{ \lambda}}\tilde{f}_t({ \theta}^{(+)},{ \lambda})- \nabla_{{ \lambda}}\tilde{f}_{t-1}({ \theta},{ \lambda}),  { \lambda}^{(+)}- { \lambda} \right\rangle}_{(I )} 			\\
		& \quad + \underbrace{\left\langle \nabla_{{ \lambda}}\tilde{f}_{t-1}({ \theta},{ \lambda})- \nabla_{{ \lambda}}\tilde{f}_{t-1}({ \theta},{ \lambda}^{(-)}),  { \lambda}-{ \lambda}^{(-)} \right\rangle}_{(II)} \\
			&\quad + \underbrace{ \left\langle \nabla_{{ \lambda}}\tilde{f}_{t-1}({ \theta},{ \lambda})- \nabla_{{ \lambda}}\tilde{f}_{t-1}({ \theta},{ \lambda}^{(-)}),{ \lambda}^{(+)}- { \lambda} - ({ \lambda}-{ \lambda}^{(-)}) \right\rangle}_{(III)} .
		\end{split}
	\end{equation*}
	\end{small}
	For $(I)$, we have
	\begin{equation*}
		\begin{split}
			& \left\langle \nabla_{{ \lambda}}\tilde{f}_t({ \theta}^{(+)},{ \lambda})- \nabla_{{ \lambda}}\tilde{f}_{t-1}({ \theta},{ \lambda}),  { \lambda}^{(+)}- { \lambda} \right\rangle \\
			&\overset{(i)}{=} \left\langle \nabla_{{ \lambda}}{f}({ \theta}^{(+)},{ \lambda})- \nabla_{{ \lambda}}{f}({ \theta},{ \lambda}),  { \lambda}^{(+)}- { \lambda} \right\rangle  \\
			& \quad + \frac{c_{t-1}-c_t}{2} \left\langle { \lambda} , { \lambda}^{(+)}-{ \lambda}\right\rangle \\
			&\overset{(ii)}{\leq}   \frac{L_{12}^2}{2a_t} \left\lVert { \theta}^{(+)} -{ \theta}\right\rVert^2 + \frac{a_t}{2} \left\lVert  { { \lambda}}^{(+)}- { \lambda}  \right\rVert^2 \\
			& \quad - \frac{c_{t-1}-c_t}{2} \left\lVert  { \lambda}^{(+)}-  { \lambda} \right\rVert^2 + \frac{c_{t-1}-c_t}{2}\left(\left\lVert { \lambda}^{(+)} \right\rVert^2 - \left\lVert { \lambda} \right\rVert^2   \right)  \\
			&{\leq}   \frac{L_{12}^2\kappa(t)}{2a_t}  \sum_{\tau=0}^{\kappa(t)-1}\left\lVert { \theta}^{(t,\tau+1)} -{ \theta}^{(t,\tau)}\right\rVert^2 + \frac{a_t}{2} \left\lVert { { \lambda}}^{(+)}- { \lambda}  \right\rVert^2\\
			& \quad + \frac{c_{t-1}-c_t}{2}\left(\left\lVert { \lambda}^{(+)} \right\rVert^2 - \left\lVert { \lambda} \right\rVert^2  \right)- \frac{c_{t-1}-c_t}{2}\left\lVert  { \lambda}^{(+)}-  { \lambda} \right\rVert^2
		\end{split}
	\end{equation*}
	where $(i)$ follows from the definition of $\tilde{f}_t$, $(ii)$ uses Assumption \ref{assum:smoothness}.
	For $(II)$, from \eqref{ineq:smooth-strong concavity} we readily have
	\begin{equation*}
		\begin{split}
			&\left\langle  \nabla_{{ \lambda}}\tilde{f}_{t-1}({ \theta},{ \lambda}) -\nabla_{{ \lambda}}\tilde{f}_{t-1}({ \theta},{ \lambda}^{(-)}), { \lambda}-{ \lambda}^{(-)}\right\rangle  \\
			&
			\leq -\frac{1}{\tilde{L}_{{ \lambda}}+c_{1}} \left\lVert  \nabla_{{ \lambda}}\tilde{f}_{t-1}({ \theta},{ \lambda}) -\nabla_{{ \lambda}}\tilde{f}_{t-1}({ \theta}, { \lambda}^{(-)})\right\rVert^2. 
		\end{split}
	\end{equation*}
	For $(III)$, we have
\begin{equation*}
	\begin{split}
		&\left\langle \nabla_{{ \lambda}}\tilde{f}_{t-1}({ \theta},{ \lambda})- \nabla_{{ \lambda}}\tilde{f}_{t-1}({ \theta},{ \lambda}^{(-)}), { \lambda}^{(+)}- { \lambda} - ({ \lambda}- { \lambda}^{(-)}) \right\rangle  \\
		& \leq  \frac{1}{2\beta}\left\lVert  \nabla_{{ \lambda}} \tilde{f}_{t-1} ({ \theta} , { \lambda}) -   \nabla_{{ \lambda}} \tilde{f}_{t-1} ({ \theta} , { \lambda}^{(-)})  \right\rVert^2 \\
		& \quad + \frac{\beta}{2} \left\lVert     { \lambda}^{(+)}- { \lambda} - \left({ \lambda}- { \lambda}^{(-)}\right) \right \rVert^2
	\end{split}
\end{equation*}
by the Cauchy-Schwartz inequality.
Upon using the above bounds for $(I)$, $(II)$ and $(III)$ and that $\frac{1}{2\beta} \leq \frac{1}{\tilde{L}_{ { \lambda}}  + c_1}$,  we obtain from \eqref{decreasing-modified-dual} that
	\begin{equation*}
		\begin{split}
		&	\tilde{f}_{t}({ \theta}^{(+)},{ \lambda}^{(+)})-\tilde{f}_t({ \theta}^{(+)},{ \lambda})  \\
		&	\leq  \frac{L_{12}^2\kappa(t)}{2a_t}  \sum_{\tau=0}^{\kappa(t)-1}\left\lVert { \theta}^{(t,\tau+1)} -{ \theta}^{(t,\tau)}\right\rVert^2 \\
		& \quad +  \frac{c_{t-1}-c_t}{2}\left(\left\lVert { \lambda}^{(+)} \right\rVert^2 - \left\lVert { \lambda} \right\rVert^2  \right)  \\
			& \quad + \frac{\beta}{2} \left\lVert { { \lambda}}- { \lambda}^{(-)}  \right\rVert^2+ \left(\frac{\beta}{2}-\frac{c_{t-1}-c_t}{2}+\frac{a_t}{2}\right) \left\lVert { { \lambda}}^{(+)}- { \lambda}  \right\rVert^2  ,
		\end{split}
	\end{equation*}
	which gives \eqref{lem:dual-progress} as desired.
\end{proof}

{\bf Lyapunuov function.} To proceed, we define a Lyapunuov function
\begin{equation*}
\begin{split}
 	F^{(t+1)} &= {f}( { \theta}^{(t+1)}, { \lambda}^{(t+1)}) + S^{(t+1)} \\
 	& \quad - {4\beta} \left\lVert   { \lambda}^{(t+1)} -  { \lambda}^{(t)}  \right\rVert^2 - \frac{c_t}{2} \left\lVert   { \lambda}^{(t+1)}  \right\rVert^2
\end{split}
\end{equation*}
where 
\begin{equation}\label{def:S}
	S^{(t+1)} = \frac{4\beta^2}{c_{t+1}}\left \lVert    { \lambda}^{(t+1)}-  { \lambda}^{(t)} \right\rVert^2 - 4\beta\left(  \frac{c_{t-1}}{c_t} -1\right) \left\lVert   { \lambda}^{(t+1)} \right\rVert^2. 
\end{equation}
For $F^{(t)}$, we have the following lemma.

\begin{lemma}[Decreasing Lyapunuov Function]\label{lem:Lyapunuov_func}
Suppose the premise of Lemma \ref{lem:dual-progress} holds. If
	\begin{equation}\label{condition:c}
	\frac{1}{c_{t+1}}-\frac{1}{c_t} \leq \frac{2}{5\beta},
\end{equation}
then, for any $t\geq 0$, it holds that
	\begin{equation}\label{Lyapunuov_function}
		\begin{split}
			& -\left(\iota_t  - \frac{L_{12}^2\kappa(t)}{2\beta} -  \frac{32\beta L_{12}^2\kappa(t)}{c_t^2} \right)\sum_{\tau=0}^{\kappa(t)-1}\left\lVert  { \theta}^{(t,\tau+1)} - { \theta}^{(t,\tau)}\right\rVert^2 \\
			&  + \frac{c_{t-1}-c_t}{2}\left\lVert  { \lambda}^{(t+1)} \right\rVert^2 - \frac{2\beta}{5}  \left\lVert  { \lambda}^{(t+1)} -  { \lambda}^{(t)}  \right\rVert^2 \\
			& + 4\beta \left( \frac{c_{t-2}}{c_{t-1}}-\frac{c_{t-1}}{c_t}  \right) \left\lVert  { \lambda}^{(t)}\right\rVert^2 \\
			& \geq F^{(t+1)} - F^{(t)}
		\end{split}
	\end{equation}
	where $\iota_t$ is defined in Lemma \ref{lem:primal}.
\end{lemma}

\begin{proof}[Proof of Lemma \ref{lem:Lyapunuov_func}]\let\qed\relax
Denote  by $\cdot^{(+)}$, $\cdot$, $\cdot^{(-)}$ the variables at time $t+1$, $t$, and $t-1$, respectively.
Similar to \eqref{dual_optimality_t}, we have
	\begin{equation}
			\left\langle \nabla_{{ \lambda}} \tilde{f}_{t}({ \theta}^{(+)} , { \lambda})- \beta ({ \lambda}^{(+)}-{ \lambda}), { \lambda} -{ \lambda}^{(+)} \right\rangle \leq 0  \label{dual_optimality_t+1}
	\end{equation}
by optimality at $t+1$.
	By \eqref{dual_optimality_t} and \eqref{dual_optimality_t+1}, we have
	\begin{equation*}
	\begin{split}
	&	\Big\langle \nabla_{{ \lambda}} \tilde{f}_t({ \theta}^{(+)} , { \lambda}) - \nabla_{{ \lambda}} \tilde{f}_{t-1}({ \theta} , { \lambda}^{(-)})- \beta  \left({ \lambda}^{(+)}- 2{ \lambda} -{ \lambda}^{(-)} \right) , \\
		& \quad \quad \quad \quad\quad \quad\quad \quad\quad \quad\quad \quad\quad \quad\quad \quad\quad \quad \quad\quad {\lambda}^{(+)} -{ \lambda} \Big\rangle \geq 0
	\end{split}
	\end{equation*}
	and therefore
	\begin{equation*}
		\begin{split}
			&\beta \left\langle   \left({ \lambda}^{(+)}- { \lambda} - \left({ \lambda}- { \lambda}^{(-)}\right) \right),    { \lambda}^{(+)}- { \lambda}  \right\rangle \\
			&\leq	\left\langle \nabla_{{ \lambda}}\tilde{f}_t({ \theta}^{(+)},{ \lambda})- \nabla_{ { \lambda}}\tilde{f}_{t-1}( { \theta}, { \lambda}),   { \lambda}^{(+)}-  { \lambda} \right \rangle \\
			& +  \left\langle \nabla_{ { \lambda}}\tilde{f}_{t-1}( { \theta}, { \lambda})- \nabla_{ { \lambda}}\tilde{f}_{t-1}( { \theta}, { \lambda}^{(-)}),  { \lambda}^{(+)}-  { \lambda} -\left ( { \lambda}-  { \lambda}^{(-)}\right) \right\rangle \\
			&  + \left\langle \nabla_{ { \lambda}}\tilde{f}_{t-1}( { \theta}, { \lambda})- \nabla_{ { \lambda}}\tilde{f}_{t-1}( { \theta}, { \lambda}^{(-)}),   { \lambda}- { \lambda}^{(-)} \right\rangle.
		\end{split}
	\end{equation*}
	Following the same line of reasoning in proving Lemma \ref{lem:dual-progress}, we have
	\begin{equation*}
		\begin{split}
			&-\frac{\beta}{2} \left \lVert  { \lambda}-  { \lambda}^{(-)} \right\rVert^2 + \frac{\beta}{2}\left\lVert       { \lambda}^{(+)} +  { \lambda}^{(-)}  \right\rVert^2 +\frac{\beta}{2}\left\lVert    { \lambda}^{(+)}-  { \lambda}\right \rVert^2 \\
			&\leq  \frac{L_{12}^2\kappa(t)}{2a_t}  \sum_{\tau=0}^{\kappa(t)-1}\left\lVert  { \theta}^{(t,\tau+1)} - { \theta}^{(t,\tau)}\right\rVert^2  + \frac{a_t}{2} \left\lVert   {  { \lambda}}^{(+)}-  { \lambda}  \right\rVert^2\\
			& \quad + \frac{c_{t-1}-c_t}{2}\left(\left\lVert  { \lambda}^{(+)} \right\rVert^2 - \left\lVert  { \lambda} \right\rVert^2  \right ) - \frac{c_{t-1}-c_t}{2}\left\lVert   { \lambda}^{(+)}-   { \lambda} \right\rVert^2 \\
			&\quad +\frac{\beta}{2}\left\lVert       { \lambda}^{(+)}-  { \lambda} - ( { \lambda}-  { \lambda}^{(-)})  \right \rVert^2- \frac{c_{t-1}\tilde{L}_{ { \lambda}}}{c_{t-1}+\tilde{L}_{ { \lambda}}} \left\lVert   { \lambda}- { \lambda}^{(-)} \right\rVert^2.
		\end{split}
	\end{equation*}
	Therefore
	\begin{equation*}
		\begin{split}
			& \frac{\beta}{2}\left\lVert    { \lambda}^{(+)}-  { \lambda} \right\rVert^2-  \frac{c_{t-1}-c_t}{2}\left\lVert  { \lambda}^{(+)} \right\rVert^2 \\
			& \leq  \left(\frac{\beta}{2}- \frac{c_{t-1}\tilde{L}_{ { \lambda}}}{c_{t-1}+\tilde{L}_{ { \lambda}}} \right) \left\lVert  { \lambda}-  { \lambda}^{(-)} \right\rVert^2 + \frac{a_t}{2}\left \lVert   {  { \lambda}}^{(+)}-  { \lambda} \right\rVert^2 \\
			& \quad  +  \frac{L_{12}^2\kappa(t)}{2a_t}  \sum_{\tau=0}^{\kappa(t)-1}\left\lVert  { \theta}^{(t,\tau+1)} - { \theta}^{(t,\tau)}\right\rVert^2      - \frac{c_{t-1}-c_t}{2}  \left\lVert  { \lambda} \right\rVert^2 .
		\end{split}
	\end{equation*}
	Since $c_1\leq \tilde{\mathcal{L}}_{ { \lambda}}$ and
$
		-\frac{c_{t-1} \tilde{\mathcal{L}}_{ { \lambda}}}{c_{t-1}+\tilde{\mathcal{L}}_{ { \lambda}}} \leq -\frac{c_{t-1}}{2}< -\frac{c_t}{2},
$
	we have
	\begin{equation*}
		\begin{split}
			& \frac{\beta}{2}  \left\lVert    { \lambda}^{(+)}-  { \lambda} \right\rVert^2-  \frac{c_{t-1}-c_t}{2}\left\lVert  { \lambda}^{(+)} \right\rVert^2 \\
			& \leq  \left(\frac{\beta}{2} - \frac{c_t}{2} \right) \left\lVert  { \lambda}-  { \lambda}^{(-)} \right\rVert^2 - \frac{c_{t-1}-c_t}{2}\left\lVert  { \lambda} \right\rVert^2  \\
			& \quad 
			+   \frac{L_{12}^2\kappa(t)}{2a_t}  \sum_{\tau=0}^{\kappa(t)-1}\left\lVert  { \theta}^{(t,\tau+1)} - { \theta}^{(t,\tau)}\right\rVert^2   + \frac{a_t}{2}\left \lVert   {  { \lambda}}^{(+)}-  { \lambda} \right \rVert^2  .
		\end{split}
	\end{equation*}
	Upon multiplying $ {8\beta}/{c_t}$ on both sides, we have
	\begin{equation*}
		\begin{split}
			& \frac{4\beta^2}{c_t}  \left\lVert    { \lambda}^{(+)}-  { \lambda} \right\rVert^2-  \frac{4\beta(c_{t-1}-c_t)}{c_t}\left\lVert  { \lambda}^{(+)} \right\rVert^2 \\
			&\leq     \frac{4\beta L_{12}^2\kappa(t)}{a_tc_t}  \sum_{\tau=0}^{\kappa(t)-1}\left\lVert  { \theta}^{(t,\tau+1)} - { \theta}^{(t,\tau)}\right\rVert^2  + \frac{4\beta a_t}{c_t}\left \lVert   {  { \lambda}}^{(+)}-  { \lambda} \right \rVert^2 \\
			& \quad  + \left(\frac{4\beta^2}{c_t} - 4\beta\right)  \left\lVert  { \lambda}-  { \lambda}^{(-)} \right\rVert^2 - \frac{4\beta(c_{t-1}-c_t)}{c_t}\left\lVert  { \lambda} \right\rVert^2 . 
		\end{split}
	\end{equation*}
	By letting $a_t={c_t}/{8}$ and using the definition of $S$ in \eqref{def:S}, we have
	\begin{equation*}
		\begin{split}
			S^{(+)} - S &\leq 4\beta \left(  \frac{c_{t-2}}{c_{t-1}}-\frac{c_{t-1}}{c_t} \right)\left\lVert  { \lambda} \right\rVert^2  - 4\beta \left\lVert  { \lambda}-  { \lambda}^{(-)} \right\rVert^2 \\
		& \quad + \frac{32\beta L_{12}^2\kappa(t)}{c_t^2} \sum_{\tau=0}^{\kappa(t)-1}\left\lVert  { \theta}^{(t,\tau+1)} - { \theta}^{(t,\tau)}\right\rVert^2  \\
			& \quad + \left(  \frac{\beta}{2}+4\beta^2\left(\frac{1}{c_{t+1}}-\frac{1}{c_t}\right) \right)\left\lVert  { \lambda}^{(+)}-  { \lambda} \right\rVert^2.
		\end{split}
	\end{equation*}
	Upon using Lemmas \ref{lem:primal} and \ref{lem:dual-progress} and letting $a_t = \beta$, we arrive at \eqref{Lyapunuov_function} as desired.
	\end{proof}


\section{Proof of Theorem \ref{thm:convergence} }
\label{app:pf-thm}

We begin by presenting an equivalent variant of Theorem \ref{thm:convergence}, where the term $\mathcal{O}(\delta^{-4})$ is expanded.
\begin{theorem}[An equivalent variant of Theorem 2]\label{thm:convergence-expanded}
	Suppose that Assumption \ref{assum:smoothness} holds. Let $\{(\theta^{(t)}, \lambda^{(t)})\}_{t\geq 0}$ be a sequence generated by Algorithms~\ref{alg:server},~\ref{alg:active}, and~\ref{alg:passive}. If $\beta \geq L_{\lambda}$, $c_t = ({\beta}{{t}^{-1/4}})/2$, and, for some $\tau >8$,
\begin{small}
	\begin{equation*}
		\begin{split}
			\eta_t \geq \frac{L^2(KQ+2)(KQ-1) + 2(L+1)}{4}+ \frac{L_{12}^2KQ (1+32\tau \sqrt{t})}{2\beta},
		\end{split}
	\end{equation*}
\end{small}
then, for any given $\delta>0$,
		\begin{equation*}
		\begin{split}
			T(\delta)  \leq\max 
			\left\{     \left(  \frac{64(\tau-8)L_{12}^2KQd_2D}{\beta \delta^2} +2\right)^2,   \frac{\beta^4\sigma^4_{\lambda}}{\delta^4} +1    \right\}
		\end{split}
	\end{equation*}
where $\sigma_{ { \lambda}}= \max_{t\geq 0}\left\{ \lVert  { \lambda^{(t)}} \rVert  \right\}$, 
\begin{equation*}
\begin{split}
D &= F^{(3)}- \underline{f} + \left( \frac{4\beta c_{1}}{c_{2}} + \frac{c_2}{2}+ 3\rho_{2}c_{2}^2    +  4(2^{1/4}+1)\beta \right) \sigma_{ { \lambda}}^2 
\end{split}
\end{equation*}
with
$\underline{{f}} = \min_{ ({ \theta},  { \lambda})}  {f}( { \theta}, { \lambda})$, and $d_2 =\max\left\{ d_1KQ, \frac{ 5\sqrt{3}\beta^2}{32(\tau -8) L_{12}^2 K}\right\}$
with 
\begin{scriptsize}
\begin{equation*}
\begin{split}
	d_1 
	&= \frac{ {2\tau^2 }}{(\tau-8)^2}+\frac{2\beta^2\left(\frac{L+1}{2}+ \frac{(KQ+2)(KQ-1)L^2}{4}+ \frac{L_{12}^2KQ}{2\beta}\right)^2 + 3\beta^2L_{12}^2}{256(\tau-8)^2L_{12}^4K^2}.
	\end{split}
\end{equation*}
\end{scriptsize}
\end{theorem}

\begin{proof}[Proof of Theorem \ref{thm:convergence-expanded}]\let\qed\relax
From $\beta \geq L_{\lambda}$ and $c_t = ({\beta}{{t}^{-1/4}})/2$, $t\geq 1$, one verifies that the conditions in \eqref{condition:beta} and \eqref{condition:c} are satisfied. For $t\geq 1$, we define	
$
			\gamma_t  = {(4\tau-32)\beta L_{12}^2\kappa(t)}/{c_t^2}
			$ with some $\tau >8$ and let
			\begin{equation*}
						\iota_t = \gamma_t + \frac{L_{12}^2\kappa(t)}{2\beta}+\frac{32L_{12}^2\beta \kappa(t)}{c_t^2}.
			\end{equation*}
Note that $\gamma_t \geq 0$ when the conditions on $\beta$, $c_t$, and $\eta_t$ hold.
Upon using Lemma \ref{lem:Lyapunuov_func}, we have
	\begin{equation}\label{modified_Lyapunuov_func}
		\begin{split}
		&	 -{\gamma}_t \sum_{\tau=0}^{\kappa(t)-1}\left\lVert  { \theta}^{(t,\tau+1)} - { \theta}^{(t,\tau)}\right\rVert^2   + 4\beta \left( \frac{c_{t-2}}{c_{t-1}}-\frac{c_{t-1}}{c_t}  \right) \left\lVert  { \lambda}^{(t)}\right\rVert^2 \\
			&  + \frac{c_{t-1}-c_t}{2}\left\lVert  { \lambda}^{(t+1)} \right\rVert^2 - \frac{2\beta}{5}  \left\lVert  { \lambda}^{(t+1)} -  { \lambda}^{(t)}  \right\rVert^2  \\
			& \geq F^{(t+1)} - F^{(t)}.
		\end{split}
	\end{equation}
	Let
	\begin{small}
	\begin{equation*}
		\nabla \tilde{G}( { \theta}^{(t)},  { \lambda}^{(t)})= \begin{pmatrix}
			{\eta_t}\left( { \theta}^{(t+1)} - { \theta}^{(t)}  \right)\\
			\frac{1}{\beta}\left(	 { \lambda}^{(t)} - \left[ { \lambda}^{(t)} + \beta\nabla_{ { \lambda}}\tilde{f}_{t-1}( { \theta}^{(t)},  { \lambda}^{(t)}) \right]_{+} \right)
		\end{pmatrix}.
	\end{equation*}
	\end{small}
Recall the definition for $	\nabla G( { \theta}^{(t)},  { \lambda}^{(t)})$ in \eqref{stationarity}.
It follows
	\begin{equation}\label{gradient-relation}
		\left\lVert  	\nabla G( { \theta}^{(t)},  { \lambda}^{(t)}) \right\rVert - 	\left\lVert  	\nabla \tilde{G}( { \theta}^{(t)},  { \lambda}^{(t)}) \right\rVert \leq c_{t-1} \left\lVert  { \lambda}^{(t)} \right\rVert.
	\end{equation}
	Due to 
$
		\left\lVert \left (\nabla \tilde{G}(\theta^{(t)},\lambda^{(t)} )\right)_{ { \theta}} \right\rVert =  {\eta_t}  \left\lVert { \theta}^{(t+1)} - { \theta}^{(t)}  \right\rVert
$
	and 
	\begin{equation*}
	\begin{split}
	&\beta \left\lVert   { \lambda}^{(t+1)} - { \lambda}^{(t)} \right\rVert + L_{12} \left\lVert  { \theta}^{(t+1)} - { \theta}^{(t)} \right\rVert  +(c_{t-1}-c_t) \left\lVert    { \lambda}^{(t)}\right\rVert \\
	& \geq \left\lVert \left (\nabla \tilde{G}(\theta^{(t)},\lambda^{(t)} ) \right)_{ { \lambda}} \right\rVert
		\end{split}
	\end{equation*}
	we have
	\begin{equation}\label{def:modified_stationarity}
		\begin{split}
		&	\left\lVert  \nabla \tilde{G} (\theta^{(t)}, \lambda^{(t)}) \right\rVert^2 \\
		&\leq \left(\eta_t^2+3L_{12}^2\right)\kappa(t)\sum_{\tau=0}^{\kappa(t)-1}  \left\lVert { \theta}^{(t,\tau+1)} - { \theta}^{(t,\tau)}  \right\rVert^2 \\
		&\quad +3(c_{t-1}-c_t)^2\left\lVert    { \lambda}^{(t)}\right\rVert^2 + 3 \beta^2 \left\lVert   { \lambda}^{(t+1)} - { \lambda}^{(t)} \right\rVert^2.
		\end{split}
	\end{equation}
	Since $\gamma_t$ and $\eta_t$ are in the same order, it follows from the definition of $d_1$ that
	\begin{equation}\label{def:d1}
	\begin{split}
		d_1 \geq \frac{\eta_t^2+3L_{12}^2}{\gamma_t^2}.
		\end{split}
	\end{equation}
%
Upon using \eqref{def:d1} and \eqref{def:modified_stationarity}, we obtain
	\begin{equation}\label{modified_stationarity}
		\begin{split}
		&	\left\lVert  \nabla \tilde{G} (\theta^{(t)}, \lambda^{(t)}) \right\rVert^2\\
		&	\leq d_1\gamma_t^2\kappa(t)\sum_{\tau=0}^{\kappa(t)-1}\left\lVert  { \theta}^{(t,\tau+1)} - { \theta}^{(t,\tau)} \right\rVert^2\\
		&\quad +3(c_{t-1}^2-c_t^2)\left\lVert    { \lambda}^{(t)}\right\rVert^2 +   3 \beta^2\left \lVert   { \lambda}^{(t+1)} - { \lambda}^{(t)} \right\rVert^2.
		\end{split}
	\end{equation}
Let $
		\rho_{t} ={1}/{(\max\{ d_1\gamma_tKQ , 15\beta/2 \})}.
$
By multiplying $\rho_t$ on both sides of \eqref{modified_stationarity} and using \eqref{modified_Lyapunuov_func}, we have
	\begin{equation*}
		\begin{split}
			&\rho_t		\left\lVert  \nabla \tilde{G} (\theta^{(t)}, \lambda^{(t)}) \right\rVert^2\\
			  &\leq  F^{(t)} - F^{(t+1)} + \frac{c_{t-1}-c_t}{2}\left\lVert  { \lambda}^{(t+1)}\right\rVert^2\\
			  &\quad +{3\rho_t(c_{t-1}^2-c_t^2)} \left \lVert    { \lambda}^{(t)}\right\rVert^2 + + 4\beta \left( \frac{c_{t-2}}{c_{t-1}}-\frac{c_{t-1}}{c_t}  \right) \left\lVert  { \lambda}^{(t)}\right\rVert^2\\
		&	\leq  F^{(t)} - F^{(t+1)} + \frac{c_{t-1}-c_t}{2}\sigma_{ { \lambda}}^2+ 4\beta \left( \frac{c_{t-2}}{c_{t-1}}-\frac{c_{t-1}}{c_t}  \right) \sigma_{ { \lambda}}^2 \\
		&\quad +{3\left(\rho_{t-1}c_{t-1}^2-\rho_tc_t^2\right)}  \sigma_{ { \lambda}}^2.
		\end{split}
	\end{equation*}
	Let
$
		\tilde{T}(\delta) = \min \left\{ t | t\geq 3, \lVert	\nabla \tilde{G}( { \theta}^{(t)},  { \lambda}^{(t)})  \rVert \leq \frac{\delta}{2}\right \}
$ and
$$\underline{F}= \min_{t\geq 3}\min_{( { \theta},  { \lambda} )} F^{(t)}.$$ By definition, we have
$
		\underline{F} \geq 	\underline{f} - {7\beta}\sigma_{ { \lambda}}^2-\left(   4\beta(2^{1/4}-1) + {\beta}\right)\sigma_{ { \lambda}}^2.
$
	It follows
	\begin{equation*}
		\begin{split}
		&	\sum_{t=3}^{	\tilde{T}(\delta) }\rho_t		\left\lVert  \nabla \tilde{G} (\theta^{(t)}, \lambda^{(t)}) \right\rVert^2  \\
		& \leq F^{(3)}- \underline{F} +4\beta \left( \frac{c_{1}}{c_{2}}-\frac{c_{	\tilde{T}(\delta)-1}}{c_{\tilde{T}(\delta)}}  \right)	\sigma_{ { \lambda}}^2 + \frac{c_2-c_{\tilde{T}(\delta)}}{2} \sigma_{ { \lambda}}^2 \\
		&\quad +  {3\left(\rho_{2}c_{2}^2-\rho_{\tilde{T}(\delta)}c_{\tilde{T}(\delta)}^2\right)} \sigma_{ { \lambda}}^2 \\
			&\leq F^{(3)}- \underline{F} + \frac{4\beta c_{1}}{c_{2}}\sigma_{ { \lambda}}^2 + \frac{c_2}{2} \sigma_{ { \lambda}}^2 + 3\rho_{2}c_{2}^2 \sigma_{ { \lambda}}^2  \\
			&\leq   F^{(3)}- \underline{f} + \frac{4\beta c_{1}}{c_{2}}\sigma_{ { \lambda}}^2 + \frac{c_2}{2} \sigma_{ { \lambda}}^2 + 3\rho_{2}c_{2}^2 \sigma_{ { \lambda}}^2 +{7\beta}\sigma_{ { \lambda}}^2\\
			& \quad +\left(   4\beta(2^{1/4}-1) + {\beta}\right)\sigma_{ { \lambda}}^2 : = D
		\end{split} 
	\end{equation*} 
By definition and $\kappa(t) \geq K$, we have $$d_2 \geq \max\left\{ d_1KQ, \frac{15\beta}{2\gamma_3}\right\}\geq \max\left\{ d_1K Q, \frac{15\beta}{2\gamma_t}\right\} = \frac{1}{\rho_t\gamma_t}.$$
Therefore,
	$\rho_t \geq {1}/{(d_2\gamma_t)}. $
	By the definition of $\tilde{T}(\delta)$, we have
	\begin{equation*}
		\frac{\delta^2}{4}\sum_{t=3}^{	\tilde{T}(\delta) }\frac{1}{\gamma_t} \leq \sum_{t=3}^{	\tilde{T}(\delta) }\frac{1}{\gamma_t} 	\left\lVert  \nabla \tilde{G} (\theta^{(t)}, \lambda^{(t)}) \right\rVert^2 \leq d_2D.
	\end{equation*}
Then
$
		{\delta^2}\leq {4d_2D}/{\left(\sum_{t=3}^{\tilde{T}(\delta)} {\gamma_t^{-1}}\right)}
$.
	By using 
$c_t = ({\beta}{{t}^{-1/4}})/2$, $
			\gamma_t  = {(4\tau-32)\beta L_{12}^2\kappa(t)}/{c_t^2}
			$, and $\kappa(t)\leq KQ$,
	we have
$$	\gamma_t \leq {16(\tau-8)L_{12}^2KQ\sqrt{t}}/{\beta}.
$$
	Since 
$
		\sum_{t=3}^{	\tilde{T}(\delta) }\frac{1}{\sqrt{t}} \geq \sqrt{\tilde{T}(\delta)}-2,
$
	we obtain
$
		\frac{\delta^2}{4}\leq \frac{16(\tau-8)L_{12}^2KQd_2D}{\beta\left(\sqrt{\tilde{T}(\delta)}-2\right)}
$
	and therefore
	\begin{equation*}
		\tilde{T}(\delta) \leq \left(  \frac{64(\tau-8)L_{12}^2KQd_2D}{\beta \delta^2} +2\right)^2.
	\end{equation*}
	In addition, if $
		t > 1+{ \beta^4 \sigma_{ { \lambda}}^4}/{\delta^4},
$
	then 
$
		c_{t-1} ={\beta}/{(2\sqrt[4]{t-1})}\leq {\delta}/{(2\sigma_{ { \lambda}})}.
$
	Therefore, according to \eqref{gradient-relation}, there exists a 
	\begin{equation*}
		\begin{split}
		&	T(\delta) \leq \max\left\{   \tilde{T}(\delta) , \frac{\beta^4\sigma_{\lambda}^4}{\delta^4} +1\right\} \\
			&\leq\max 
			\left\{     \left(  \frac{64(\tau-8)L_{12}^2KQd_2D}{\beta \delta^2} +2\right)^2,   \frac{\beta^4\sigma^4_{\lambda}}{\delta^4} +1    \right\}
		\end{split}
	\end{equation*}
	such that $
	\left\lVert  	\nabla G( { \theta}^{(t)},  { \lambda}^{(t)}) \right\rVert  \leq  	\left\lVert  	\nabla \tilde{G}( { \theta}^{(t)},  { \lambda}^{(t)}) \right\rVert + c_{t-1} \left\lVert  { \lambda}^{(t)} \right\rVert \leq \delta
	$. This completes the proof.
\end{proof}

\bibliographystyle{IEEEtran}
\bibliography{reference}

\begin{thebibliography}{10}
\providecommand{\url}[1]{#1}
\csname url@samestyle\endcsname
\providecommand{\newblock}{\relax}
\providecommand{\bibinfo}[2]{#2}
\providecommand{\BIBentrySTDinterwordspacing}{\spaceskip=0pt\relax}
\providecommand{\BIBentryALTinterwordstretchfactor}{4}
\providecommand{\BIBentryALTinterwordspacing}{\spaceskip=\fontdimen2\font plus
\BIBentryALTinterwordstretchfactor\fontdimen3\font minus
  \fontdimen4\font\relax}
\providecommand{\BIBforeignlanguage}[2]{{%
\expandafter\ifx\csname l@#1\endcsname\relax
\typeout{** WARNING: IEEEtran.bst: No hyphenation pattern has been}%
\typeout{** loaded for the language `#1'. Using the pattern for}%
\typeout{** the default language instead.}%
\else
\language=\csname l@#1\endcsname
\fi
#2}}
\providecommand{\BIBdecl}{\relax}
\BIBdecl

\bibitem{mcmahan2017communication}
B.~McMahan, E.~Moore, D.~Ramage, S.~Hampson, and B.~A. y~Arcas,
  ``Communication-efficient learning of deep networks from decentralized
  data,'' in \emph{Artificial intelligence and statistics}.\hskip 1em plus
  0.5em minus 0.4em\relax PMLR, 2017, pp. 1273--1282.

\bibitem{hardy2017private}
S.~Hardy, W.~Henecka, H.~Ivey-Law, R.~Nock, G.~Patrini, G.~Smith, and
  B.~Thorne, ``Private federated learning on vertically partitioned data via
  entity resolution and additively homomorphic encryption,'' \emph{arXiv
  preprint arXiv:1711.10677}, 2017.

\bibitem{sun2019privacy}
C.~Sun, L.~Ippel, J.~Van~Soest, B.~Wouters, A.~Malic, O.~Adekunle, B.~van~den
  Berg, O.~Mussmann, A.~Koster, C.~van~der Kallen \emph{et~al.}, ``A
  privacy-preserving infrastructure for analyzing personal health data in a
  vertically partitioned scenario.'' in \emph{MedInfo}, 2019, pp. 373--377.

\bibitem{yang2019federated}
Q.~Yang, Y.~Liu, T.~Chen, and Y.~Tong, ``Federated machine learning: Concept
  and applications,'' \emph{ACM Transactions on Intelligent Systems and
  Technology (TIST)}, vol.~10, no.~2, pp. 1--19, 2019.

\bibitem{caton2020fairness}
S.~Caton and C.~Haas, ``Fairness in machine learning: A survey,'' \emph{arXiv
  preprint arXiv:2010.04053}, 2020.

\bibitem{pessach2020algorithmic}
D.~Pessach and E.~Shmueli, ``Algorithmic fairness,'' \emph{arXiv preprint
  arXiv:2001.09784}, 2020.

\bibitem{liu2019communication}
Y.~Liu, Y.~Kang, X.~Zhang, L.~Li, Y.~Cheng, T.~Chen, M.~Hong, and Q.~Yang, ``A
  communication efficient collaborative learning framework for distributed
  features,'' \emph{arXiv preprint arXiv:1912.11187}, 2019.

\bibitem{zhang2021secure}
Q.~Zhang, B.~Gu, C.~Deng, and H.~Huang, ``Secure bilevel asynchronous vertical
  federated learning with backward updating,'' in \emph{Proceedings of the AAAI
  Conference on Artificial Intelligence}, vol.~35, no.~12, 2021, pp.
  10\,896--10\,904.

\bibitem{calders2009building}
T.~Calders, F.~Kamiran, and M.~Pechenizkiy, ``Building classifiers with
  independency constraints,'' in \emph{2009 IEEE International Conference on
  Data Mining Workshops}.\hskip 1em plus 0.5em minus 0.4em\relax IEEE, 2009,
  pp. 13--18.

\bibitem{hardt2016equality}
M.~Hardt, E.~Price, and N.~Srebro, ``Equality of opportunity in supervised
  learning,'' \emph{Advances in neural information processing systems},
  vol.~29, pp. 3315--3323, 2016.

\bibitem{donini2018empirical}
M.~Donini, L.~Oneto, S.~Ben-David, J.~Shawe-Taylor, and M.~Pontil, ``Empirical
  risk minimization under fairness constraints,'' \emph{arXiv preprint
  arXiv:1802.08626}, 2018.

\bibitem{calmon2017optimized}
F.~P. Calmon, D.~Wei, B.~Vinzamuri, K.~N. Ramamurthy, and K.~R. Varshney,
  ``Optimized pre-processing for discrimination prevention,'' in
  \emph{Proceedings of the 31st International Conference on Neural Information
  Processing Systems}, 2017, pp. 3995--4004.

\bibitem{feldman2015certifying}
M.~Feldman, S.~A. Friedler, J.~Moeller, C.~Scheidegger, and
  S.~Venkatasubramanian, ``Certifying and removing disparate impact,'' in
  \emph{proceedings of the 21th ACM SIGKDD international conference on
  knowledge discovery and data mining}, 2015, pp. 259--268.

\bibitem{kamiran2012data}
F.~Kamiran and T.~Calders, ``Data preprocessing techniques for classification
  without discrimination,'' \emph{Knowledge and Information Systems}, vol.~33,
  no.~1, pp. 1--33, 2012.

\bibitem{luong2011k}
B.~T. Luong, S.~Ruggieri, and F.~Turini, ``k-nn as an implementation of
  situation testing for discrimination discovery and prevention,'' in
  \emph{Proceedings of the 17th ACM SIGKDD international conference on
  Knowledge discovery and data mining}, 2011, pp. 502--510.

\bibitem{corbett2017algorithmic}
S.~Corbett-Davies, E.~Pierson, A.~Feller, S.~Goel, and A.~Huq, ``Algorithmic
  decision making and the cost of fairness,'' in \emph{Proceedings of the 23rd
  acm sigkdd international conference on knowledge discovery and data mining},
  2017, pp. 797--806.

\bibitem{dwork2018decoupled}
C.~Dwork, N.~Immorlica, A.~T. Kalai, and M.~Leiserson, ``Decoupled classifiers
  for group-fair and efficient machine learning,'' in \emph{Conference on
  fairness, accountability and transparency}.\hskip 1em plus 0.5em minus
  0.4em\relax PMLR, 2018, pp. 119--133.

\bibitem{menon2018cost}
A.~K. Menon and R.~C. Williamson, ``The cost of fairness in binary
  classification,'' in \emph{Conference on Fairness, Accountability and
  Transparency}.\hskip 1em plus 0.5em minus 0.4em\relax PMLR, 2018, pp.
  107--118.

\bibitem{pleiss2017fairness}
G.~Pleiss, M.~Raghavan, F.~Wu, J.~Kleinberg, and K.~Q. Weinberger, ``On
  fairness and calibration,'' \emph{arXiv preprint arXiv:1709.02012}, 2017.

\bibitem{agarwal2018reductions}
A.~Agarwal, A.~Beygelzimer, M.~Dud{\'\i}k, J.~Langford, and H.~Wallach, ``A
  reductions approach to fair classification,'' in \emph{International
  Conference on Machine Learning}.\hskip 1em plus 0.5em minus 0.4em\relax PMLR,
  2018, pp. 60--69.

\bibitem{agarwal2019fair}
A.~Agarwal, M.~Dud{\'\i}k, and Z.~S. Wu, ``Fair regression: Quantitative
  definitions and reduction-based algorithms,'' in \emph{International
  Conference on Machine Learning}.\hskip 1em plus 0.5em minus 0.4em\relax PMLR,
  2019, pp. 120--129.

\bibitem{kamiran2010discrimination}
F.~Kamiran, T.~Calders, and M.~Pechenizkiy, ``Discrimination aware decision
  tree learning,'' in \emph{2010 IEEE International Conference on Data
  Mining}.\hskip 1em plus 0.5em minus 0.4em\relax IEEE, 2010, pp. 869--874.

\bibitem{kamishima2012fairness}
T.~Kamishima, S.~Akaho, H.~Asoh, and J.~Sakuma, ``Fairness-aware classifier
  with prejudice remover regularizer,'' in \emph{Joint European Conference on
  Machine Learning and Knowledge Discovery in Databases}.\hskip 1em plus 0.5em
  minus 0.4em\relax Springer, 2012, pp. 35--50.

\bibitem{quadrianto2019discovering}
N.~Quadrianto, V.~Sharmanska, and O.~Thomas, ``Discovering fair representations
  in the data domain,'' in \emph{Proceedings of the IEEE/CVF Conference on
  Computer Vision and Pattern Recognition}, 2019, pp. 8227--8236.

\bibitem{kairouz2019advances}
P.~Kairouz, H.~B. McMahan, B.~Avent, A.~Bellet, M.~Bennis, A.~N. Bhagoji,
  K.~Bonawitz, Z.~Charles, G.~Cormode, R.~Cummings \emph{et~al.}, ``Advances
  and open problems in federated learning,'' \emph{arXiv preprint
  arXiv:1912.04977}, 2019.

\bibitem{gascon2016secure}
A.~Gasc{\'o}n, P.~Schoppmann, B.~Balle, M.~Raykova, J.~Doerner, S.~Zahur, and
  D.~Evans, ``Secure linear regression on vertically partitioned datasets.''
  \emph{IACR Cryptol. ePrint Arch.}, vol. 2016, p. 892, 2016.

\bibitem{gong2016private}
Y.~Gong, Y.~Fang, and Y.~Guo, ``Private data analytics on biomedical sensing
  data via distributed computation,'' \emph{IEEE/ACM transactions on
  computational biology and bioinformatics}, vol.~13, no.~3, pp. 431--444,
  2016.

\bibitem{zhang2018feature}
G.-D. Zhang, S.-Y. Zhao, H.~Gao, and W.-J. Li, ``Feature-distributed svrg for
  high-dimensional linear classification,'' \emph{arXiv preprint
  arXiv:1802.03604}, 2018.

\bibitem{hu2019fdml}
Y.~Hu, D.~Niu, J.~Yang, and S.~Zhou, ``Fdml: A collaborative machine learning
  framework for distributed features,'' in \emph{Proceedings of the 25th ACM
  SIGKDD International Conference on Knowledge Discovery \& Data Mining}, 2019,
  pp. 2232--2240.

\bibitem{chen2020vafl}
T.~Chen, X.~Jin, Y.~Sun, and W.~Yin, ``Vafl: a method of vertical asynchronous
  federated learning,'' \emph{arXiv preprint arXiv:2007.06081}, 2020.

\bibitem{mohri2019agnostic}
M.~Mohri, G.~Sivek, and A.~T. Suresh, ``Agnostic federated learning,'' in
  \emph{International Conference on Machine Learning}.\hskip 1em plus 0.5em
  minus 0.4em\relax PMLR, 2019, pp. 4615--4625.

\bibitem{du2021fairness}
W.~Du, D.~Xu, X.~Wu, and H.~Tong, ``Fairness-aware agnostic federated
  learning,'' in \emph{Proceedings of the 2021 SIAM International Conference on
  Data Mining (SDM)}.\hskip 1em plus 0.5em minus 0.4em\relax SIAM, 2021, pp.
  181--189.

\bibitem{zhang2019asyspa}
J.~Zhang and K.~You, ``Asyspa: An exact asynchronous algorithm for convex
  optimization over digraphs,'' \emph{IEEE Transactions on Automatic Control},
  vol.~65, no.~6, pp. 2494--2509, 2019.

\bibitem{komiyama2018nonconvex}
J.~Komiyama, A.~Takeda, J.~Honda, and H.~Shimao, ``Nonconvex optimization for
  regression with fairness constraints,'' in \emph{International conference on
  machine learning}.\hskip 1em plus 0.5em minus 0.4em\relax PMLR, 2018, pp.
  2737--2746.

\bibitem{luenberger1984linear}
D.~G. Luenberger and Y.~Ye, \emph{Linear and nonlinear programming}.\hskip 1em
  plus 0.5em minus 0.4em\relax Springer, 1984, vol.~2.

\bibitem{bertsekas1997nonlinear}
D.~P. Bertsekas, ``Nonlinear programming,'' \emph{Journal of the Operational
  Research Society}, vol.~48, no.~3, pp. 334--334, 1997.

\bibitem{xu2020unified}
Z.~Xu, H.~Zhang, Y.~Xu, and G.~Lan, ``A unified single-loop alternating
  gradient projection algorithm for nonconvex-concave and convex-nonconcave
  minimax problems,'' \emph{arXiv preprint arXiv:2006.02032}, 2020.

\bibitem{mcmahan2016federated}
H.~B. McMahan, E.~Moore, D.~Ramage, and B.~A. y~Arcas, ``Federated learning of
  deep networks using model averaging,'' \emph{arXiv preprint
  arXiv:1602.05629}, 2016.

\bibitem{cheng2021secureboost}
K.~Cheng, T.~Fan, Y.~Jin, Y.~Liu, T.~Chen, D.~Papadopoulos, and Q.~Yang,
  ``Secureboost: A lossless federated learning framework,'' \emph{IEEE
  Intelligent Systems}, vol.~36, no.~6, pp. 87--98, 2021.

\bibitem{xu2019hybridalpha}
R.~Xu, N.~Baracaldo, Y.~Zhou, A.~Anwar, and H.~Ludwig, ``Hybridalpha: An
  efficient approach for privacy-preserving federated learning,'' in
  \emph{Proceedings of the 12th ACM Workshop on Artificial Intelligence and
  Security}, 2019, pp. 13--23.

\bibitem{bezanson2017julia}
J.~Bezanson, A.~Edelman, S.~Karpinski, and V.~B. Shah, ``Julia: A fresh
  approach to numerical computing,'' \emph{SIAM review}, vol.~59, no.~1, pp.
  65--98, 2017.

\bibitem{kohavi1996scaling}
R.~Kohavi \emph{et~al.}, ``Scaling up the accuracy of naive-bayes classifiers:
  A decision-tree hybrid.'' in \emph{Kdd}, vol.~96, 1996, pp. 202--207.

\bibitem{compas}
{Larson, Jeff and Mattu, Surya and Kirchner, Lauren and Angwin, Julia}, ``How
  we analyzed the compas recidivism algorithm,''
  \url{https://www.propublica.org/article/how-we-analyzed-the-compas-recidivism-algorithm},
  2016.

\bibitem{redmond2002data}
M.~Redmond and A.~Baveja, ``A data-driven software tool for enabling
  cooperative information sharing among police departments,'' \emph{European
  Journal of Operational Research}, vol. 141, no.~3, pp. 660--678, 2002.

\bibitem{nesterov2003introductory}
Y.~Nesterov, \emph{Introductory lectures on convex optimization: A basic
  course}.\hskip 1em plus 0.5em minus 0.4em\relax Springer Science \& Business
  Media, 2003, vol.~87.

\end{thebibliography}

\end{document}